\documentclass[twoside]{article}
\usepackage[margin=.9in]{geometry}
\usepackage{microtype}
\usepackage{graphicx}
\usepackage{subfigure}
\usepackage{booktabs} 
\usepackage{natbib}
\usepackage{amsmath,amsthm,amsfonts,amssymb}
\usepackage{bbm}
\usepackage{mathrsfs}
\usepackage{enumitem}
\usepackage{bm}
\usepackage{booktabs}
\usepackage{makecell}
\usepackage{graphicx}
\usepackage{subfigure}
\usepackage{caption}
\usepackage{thmtools}
\usepackage{times}

\usepackage{tikz}

\newcommand{\Bern}{\text{Bern}}

\newcommand{\hstar}{\h^{\star}}
\newcommand{\fstar}{\f^\star}
\newcommand{\fjstar}{f_{j}^\star}
\newcommand{\lprp}[1]{\left(#1\right)}
\newcommand{\tepsilon}{\tilde{\epsilon}}
\newcommand{\balphastar}{\balpha^{\star}}
\newcommand{\balphastarj}{\balpha^{\star}_{j}}
\newcommand{\Bstar}{\B^{\star}}
\newcommand{\cFt}{\cF^{\otimes t}}

\newcommand{\EMPGaus}[2]{\hat{\mathfrak{G}}_{#1}{(#2)}}
\newcommand{\POPGaus}[2]{\mathfrak{G}_{#1}{(#2)}}
\newcommand{\MAXGaus}[2]{\bar{\mathfrak{G}}_{#1}{(#2)}}
\newcommand{\dist}[2]{d_{\cF, \cF_0}(#1; #2)}

\newcommand{\EMPRad}[2]{\hat{\mathfrak{R}}_{#1}{(#2)}}
\newcommand{\POPRad}[2]{\mathfrak{R}_{#1}{(#2)}}

\newcommand{\avdist}[3]{\bar{d}_{\cF, #1}(#2; #3)}

\newcommand{\bX}{\bar{\X}}

\newcommand{\nbf}{\mathbf{f}}
\newcommand{\hbf}{\hat{\mathbf{f}}}
\newcommand{\tbf}{\tilde{\mathbf{f}}}

\newcommand{\hLtrain}{\hat{R}_{\text{train}}}
\newcommand{\hLtest}{\hat{R}_{\text{test}}}
\newcommand{\Ltrain}{R_\text{train}}
\newcommand{\Ltest}{R_\text{test}}

\newcommand{\cQ}{\mathcal{Q}}
\newcommand{\q}{\mathbf{q}}

\newcommand{\mR}{\mathbb{R}}

\newcommand{\hf}{\hat{f}}

\newcommand{\tf}{\tilde{f}}

\newcommand{\bb}{\mathbf{b}}
\newcommand{\bepsilon}{\bm{\epsilon}}

\newcommand{\tnu}{\tilde{\nu}}

\newcommand{\C}{\mathbf{C}}
\newcommand{\hB}{\hat{\B}}

\newcommand{\halpha}{\hat{\balpha}}

\newcommand{\KL}{\text{KL}}

\newcommand{\balpha}{\bm{\alpha}}

\usepackage[colorlinks,citecolor=blue,urlcolor=black,linkcolor=blue,pdfborder={0 0 0}]{hyperref}

\usepackage[capitalize]{cleveref}  
\crefname{nlem}{Lemma}{Lemmas}
\crefname{nprop}{Proposition}{Propositions}
\crefname{ncor}{Corollary}{Corollaries}
\crefname{nthm}{Theorem}{Theorems}
\crefname{exa}{Example}{Examples}
\crefname{assumption}{Assumption}{Assumptions}
\crefname{equation}{}{}

\theoremstyle{plain}
\newtheorem{theorem}{Theorem}

\newtheorem{corollary}{Corollary}
\newtheorem{assumption}{Assumption}
\newtheorem{lemma}{Lemma}

\theoremstyle{definition}

\newtheorem{definition}{Definition}

\newtheorem{remark}{Remark}

\newcommand{\vect}[1]{\ensuremath{\mathbf{#1}}}

\newcommand{\mat}[1]{\ensuremath{\mathbf{#1}}}

\newcommand{\argmin}{\mathop{\rm argmin}}

\newcommand{\tr}{\mathrm{tr}}

\newcommand{\trans}{^{\top}}
\newcommand{\poly}{\mathrm{poly}}

\newcommand{\abs}[1]{|{#1}|}
\newcommand{\norm}[1]{\|{#1} \|}

\newcommand{\mE}{\mathbb{E}}
\renewcommand{\Pr}{\mathbb{P}}
\newcommand{\Var}{\text{Var}}

\newcommand{\tlO}{\tilde{O}}

\newcommand{\A}{\mat{A}}
\newcommand{\B}{\mat{B}}
\newcommand{\Bone}{\hat{\mat{B}}}

\newcommand{\W}{\mat{W}}
\newcommand{\X}{\mat{X}}

\newcommand{\Z}{\mat{Z}}

\newcommand{\F}{\mat{F}}

\newcommand{\mSigma}{\mat{\Sigma}}

\renewcommand{\v}{\vect{v}}

\newcommand{\h}{\vect{h}}
\newcommand{\hh}{\vect{\hat{h}}}

\newcommand{\x}{\vect{x}}
\renewcommand{\r}{\vect{r}}
\renewcommand{\b}{\vect{b}}
\newcommand{\y}{\vect{y}}
\newcommand{\z}{\vect{z}}
\newcommand{\g}{\vect{g}}
\newcommand{\f}{\vect{f}}

\newcommand{\cF}{\mathcal{F}}

\newcommand{\cN}{\mathcal{N}}

\newcommand{\cH}{\mathcal{H}}

\newcommand{\cW}{\mathcal{W}}
\newcommand{\cX}{\mathcal{X}}
\newcommand{\cY}{\mathcal{Y}}
\newcommand{\cZ}{\mathcal{Z}}

\usepackage[utf8]{inputenc} 
\usepackage{hyperref}       
\usepackage{url}            
\usepackage{nicefrac}       

\title{\textbf{On the Theory of Transfer Learning: The Importance of Task Diversity}}

\author{
Nilesh Tripuraneni \\
University of California, Berkeley\\
\texttt{nilesh\_tripuraneni@berkeley.edu}\\
\and
Michael I. Jordan\\
University of California, Berkeley\\
\texttt{jordan@cs.berkeley.edu} \\
\and
Chi Jin\\
Princeton University\\
\texttt{chij@princeton.edu}\\
}
\date{}
\begin{document}

\maketitle

\begin{abstract}
We provide new statistical guarantees for transfer learning via representation learning--when transfer is achieved by learning a feature representation shared across different tasks. This enables learning on new tasks using far less data than is required to learn them in isolation. Formally, we consider $t+1$ tasks parameterized by functions of the form $f_j \circ h$ in a general function class $\cF \circ \cH$, where each $f_j$ is a task-specific function in $\cF$ and $h$ is the shared representation in $\cH$. Letting $C(\cdot)$ denote the complexity measure of the function class, we show that for diverse training tasks (1) the sample complexity needed to learn the shared representation across the first $t$ training tasks scales as $C(\cH) + t C(\cF)$, despite no explicit access to a signal from the feature representation and (2) with an accurate estimate of the representation, the sample complexity needed to learn a new task scales only with $C(\cF)$. Our results depend upon a new general notion of task diversity--applicable to models with general tasks, features, and losses--as well as a novel chain rule for Gaussian complexities. Finally, we exhibit the utility of our general framework in several models of importance in the literature.
\end{abstract}

\section{Introduction}
\label{sec:intro}
Transfer learning is quickly becoming an essential tool to address learning problems in settings with \textit{small} data. One of the most promising methods for multitask and transfer learning is founded on the belief that multiple, differing tasks are distinguished by a small number of task-specific parameters, but often share a common low-dimensional representation. Undoubtedly, one of the most striking successes of this idea has been to only re-train the final layers of a neural network on new task data, after initializing its earlier layers with hierarchical representations/features from ImageNet (i.e., ImageNet pretraining) \citep{donahue2014decaf, gulshan2016development}. However, the practical purview of transfer learning has extended far beyond the scope of computer vision and classical ML application domains such as deep reinforcement learning \citep{baevski2019cloze}, to problems such as protein engineering and design \citep{Elnaggar864405}.

In this paper, we formally study the composite learning model in which there are $t+1$ tasks whose responses are generated noisily from the function $\fjstar \circ \hstar$, where $\fjstar$ are task-specific parameters in a function class $\cF$ and $\hstar$ an underlying shared representation in a function class $\cH$. A large empirical literature has documented the performance gains that can be obtained by transferring a jointly learned representation $\h$ to new tasks in this model~\citep{yosinski2014transferable, raghu2019rapid, lee2019meta}. There is also a theoretical literature that dates back at least as far as \citep{baxter2000model}.
However, this progress belies a lack of understanding of the basic statistical principles underlying transfer learning\footnote{A problem which is also often referred to as learning-to-learn (LTL).}:
\begin{center}
    \textbf{How many samples do we need to learn a feature representation shared across tasks and use it to improve prediction on a new task?}
\end{center}
In this paper we study a simple two-stage empirical risk minimization procedure to learn a new, $j=0$th task which shares a common representation with $t$ different training tasks. This procedure first learns a representation $\hh \approx \hstar$ given $n$ samples from each of $t$ different training tasks, and then uses $\hh$  alongside $m$ fresh samples from this new task to learn $\hat{f}_0 \circ \hh \approx f^{\star}_0 \circ \hstar$.
Informally, our main result provides an answer to our sampling-complexity question by showing that the excess risk of prediction of this two-stage procedure scales (on the new task) as\footnote{See \cref{cor:ltlguarantee} and discussion for a formal statement. Note our guarantees also hold for nonparametric function classes, but the scaling with $n$, $t$, $m$ may in general be different.},
	\begin{align*}
		\tlO \left( \frac{1}{\nu} \left(
		\sqrt{\frac{C(\cH)+t C(\cF)}{nt}} \right) + \sqrt{\frac{C(\cF)}{m}}  \right),
	\end{align*}
where $C(\cH)$ captures the complexity of the shared representation,  $C(\cF)$ captures the complexity of the task-specific maps, and $\nu$ encodes a problem-agnostic notion of task diversity.  The latter is a key contribution of the current paper.  It represents the extent to which the $t$ training tasks $\fjstar$ cover the space of the features $\hstar$. In the limit that $n, t \to \infty$ (i.e., training task data is abundant), to achieve a fixed level of constant prediction error on the new task only requires the number of fresh samples to be $m \approx C(\cF)$. Learning the task in isolation suffers the burden of learning both $\cF$ and $\cH$---requiring $m \approx C(\cF\circ \cH)$---which can be significantly greater than the transfer learning sample complexity.

\citet{maurer2016benefit} present a general, uniform-convergence based framework for obtaining generalization bounds for transfer learning that scale as $ O(1/\sqrt{t}) + O(1/\sqrt{m})$ (for clarity we have suppressed complexity factors in the numerator). Perhaps surprisingly, the leading term capturing the complexity of learning $\hstar$ decays only in $t$ but not in $n$.  This suggests that increasing the number of samples per training task cannot improve generalization on new tasks. Given that most transfer learning applications in the literature collect information from only a few training tasks (i.e., $n \gg t$), this result does not provide a fully satisfactory explanation for the practical efficacy of transfer learning methods. 

Our principal contributions in this paper are as follows:
\begin{itemize}[leftmargin=.5cm]
	\item We introduce a problem-agnostic definition of task diversity which can be integrated into a uniform convergence framework to provide generalization bounds for transfer learning problems with general losses, tasks, and features. Our framework puts this notion of diversity together with a common-design assumption across tasks to provide guarantees of a fast convergence rate, decaying with \textit{all of the samples} for the transfer learning problem. 
	\item We provide general-purpose bounds which decouple the complexity of learning the task-specific structure from the complexity of learning the shared feature representation. Our results repose on a novel user-friendly chain rule for Gaussian processes which may be of independent interest (see \cref{thm:chain_rule}). 
	Crucially, this chain rule implies a form of modularity that allows us to exploit a plethora of existing results from the statistics and machine learning literatures to individually bound the sample complexity of learning task and feature functions. 
	\item We highlight the utility of our framework for obtaining end-to-end transfer learning guarantees for several different multi-task learning models including (1) logistic regression, (2) deep neural network regression, and (3) robust regression for single-index models.
	\end{itemize}
	
\subsection{Related Work}
\label{sec:related}
The utility of multitask learning methods was observed at least as far back as \citet{caruana1997multitask}. In recent years, representation learning, transfer learning, and meta-learning have been the subject of extensive empirical investigation in the machine learning literature (see \citep{bengio2013representation}, \citep{hospedales2020meta} for surveys in these directions). However, theoretical work on transfer learning---particularly via representation learning---has been much more limited.

A line of work closely related to transfer learning is
gradient-based meta-learning (MAML) \citep{finn2017model}. These methods have been analyzed using techniques from online convex optimization, using a (potentially data-dependent) notion of task similarity which assumes that tasks are close to a global task parameter \citep{finn2019online, khodak2019provable, denevi2019learning,denevi2019online, khodak2019adaptive}. Additionally,  \citet{ben2008notion} define a different notion of distributional task similarity they use to show generalization bounds. However, these works do not study the question of transferring a common representation in the generic composite learning model that is our focus.

In settings restricted to linear task mappings and linear features, \citet{lounici2011oracle}, \citet{pontil2013excess}, and \citet{cavallanti2010linear} have provided sample complexity bounds for the problem of transfer learning via representation learning. \citet{lounici2011oracle} and \citet{obozinski2011support} also address sparsity-related issues that can arise in linear feature learning.

To our knowledge, \citet{baxter2000model} is the first theoretical work to provide generalization bounds for transfer learning via representation learning in a general setting. The formulation of  \citet{baxter2000model} assumes a generative model over tasks which share common features; in our setting, this task generative model is replaced by the assumption that training tasks are diverse (as in \cref{def:diversity}) and that there is a common covariate distribution across different tasks. In follow-up work, \citet{maurer2016benefit} propose a general, uniform-convergence-based framework for obtaining transfer learning guarantees which scale as $O(1/\sqrt{t}) + O(1/\sqrt{m})$  \citep[Theorem 5]{maurer2016benefit}. The second term represents the sample complexity of learning in a lower-dimensional space given the common representation. The first term is the bias contribution from transferring the representation---learned from an aggregate of $nt$ samples across different training tasks---to a new task. Note this leading term decays only in $t$ and not in $n$: implying that increasing the number of samples per training task cannot improve generalization on new tasks. Unfortunately, under the framework studied in that paper, this $\Omega(1/\sqrt{t})$ cannot be improved \citet{maurer2016benefit}.

Recent work in \citet{tripuraneni2020provable} and \citet{du2020few} has shown that in specific settings leveraging (1) common design assumptions across tasks and (2) a particular notion of task diversity, can break this barrier and yield rates for the leading term which decay as $O(\poly(1/(nt)))$. However, the results and techniques used in both of these works are limited to the squared loss and linear task maps.
Moreover, the notion of diversity in both cases arises purely from the linear-algebraic conditioning of the set of linear task maps. It is not clear from these works how to extend these ideas/techniques beyond the case-specific analyses therein.
\section{Preliminaries}
\label{sec:prelim}
\textbf{Notation: } We use bold lower-case letters (e.g., $\x$) to refer to vectors and bold upper-case letters (e.g., $\X$) to refer to matrices. The norm $\Vert \cdot \Vert$ appearing on a vector or matrix refers to its $\ell_2$ norm or spectral norm respectively. We use the bracketed notation $[n] = \{1, \hdots, n\}$ as shorthand for integer sets. Generically, we will use ``hatted" vectors and matrices (e.g, $\halpha$ and $\Bone$) to refer to (random) estimators of their underlying population quantities. $\sigma_1(\A), \hdots, \sigma_r(\A)$ will denote the sorted singular values (in decreasing magnitude) of a rank $r$ matrix $\A$. Throughout we will use $\cF$ to refer to a function class of tasks mapping $ \mR^r \to \mR$ and $\cH$ to be a function class of features mapping $\mR^d \to \mR^r$. For the function class $\cF$, we use $\cF^{\otimes t}$ to refer its $t$-fold Cartesian product, i.e., $\cF^{\otimes t} = \{\f \equiv (f_1, \ldots, f_t) ~|~  f_j \in \cF \text{~for any~} j \in [t]\}$. We use $\tlO$ to denote an expression that hides polylogarithmic factors in all problem parameters. 
\subsection{Transfer learning with a shared representation}
In our treatment of transfer learning, we assume that there exists a generic nonlinear feature representation that is shared across all tasks. Since this feature representation is shared, it can be utilized to transfer knowledge from existing tasks to new tasks. Formally, we assume that for a particular task $j$, we observe multiple data pairs $\{(\x_{ji}, y_{ji})\}$ (indexed over $i$) that are sampled i.i.d from an \emph{unknown} distribution $\Pr_j$, supported over $\cX \times \cY$ and defined as follows: 
\begin{equation}
	 \Pr_{j}(\x, y) = \Pr_{f^\star_j \circ \h^\star} (\x, y) = 
	 \Pr_{\x}(\x) \Pr_{y | \x}(y | f^\star_j \circ \h^\star(\x)).
	 \label{eq:gen_model}
\end{equation}

Here, $\hstar : \mR^d  \to \mR^r$ is the shared feature representation, and $\fjstar : \mR^r \to \mR$ is a task-specific mapping. Note that we assume that the marginal distribution over $\cX$---$\Pr_\x$---is common amongst all the tasks.

We consider transfer learning methods consisting of two phases. In the first phase (the training phase), $t$ tasks with $n$ samples per task are available for learning. Our objective in this phase is to learn the shared feature representation using the entire set of $nt$ samples from the first $j \in [t]$ tasks. In the second phase (the test phase), we are presented with $m$ fresh samples from a new task that we denote as the $0$th task. 
Our objective in the test phase is to learn this new task based on both the fresh samples and the representation learned in the first phase. 

Formally, we consider a two-stage Empirical Risk Minimization (ERM) procedure for transfer learning. Consider a function class $\cF$ containing task-specific functions, and a function class $\cH$ containing feature maps/representations. In the training phase, the empirical risk for $t$ training tasks is:
\begin{equation}
\hLtrain(\nbf, \h) := \frac{1}{n t} \sum_{j=1}^{t} \sum_{i=1}^{n} \ell(f_{j} \circ \h(\x_{ji}), y_{ji}), \label{eq:metatrainerm}
\end{equation}
where $\ell(\cdot, \cdot)$ is the loss function and $\nbf:= (f_1, \hdots, f_t) \in \cF^{\otimes t}$. Our estimator $\hat{\h}(\cdot)$ for the shared data representation is given by $\hat{\h} = \argmin_{\h \in \cH} \min_{\nbf \in \cFt}  \hLtrain(\nbf, \h)$.
 
For the second stage, the empirical risk for learning the new task is defined as:
\begin{equation}
\hLtest(f, \h) := \frac{1}{m} \sum_{i=1}^{m} \ell(f \circ \h(\x_{0i}), y_{0i}). \label{eq:metatesterm}
\end{equation}
We estimate the underlying function $f_0^\star$ for task $0$ by computing the ERM based on the feature representation learned in the first phase. That is, $\hf_0 = \argmin_{f \in \cF} \hLtest(f, \hh)$. We gauge the efficacy of the estimator $(\hat{f}_0, \hat{\h})$ by its excess risk on the new task, which we refer to as the \emph{transfer learning risk}:
\begin{align}
\text{Transfer Learning Risk} = \Ltest(\hf_{0}, \hh) - \Ltest(f_0^{\star}, \hstar).
\label{eq:excessrisk}
\end{align}
Here, $\Ltest(\cdot, \cdot) = \mE[\hLtest(\cdot, \cdot)]$ is the population risk for the new task and the population risk over the $t$ training tasks is similarly defined as $\Ltrain(\cdot, \cdot) = \mE[\hLtrain(\cdot, \cdot)]$; both expectations are taken over the randomness in the training and test phase datasets respectively. The transfer learning risk measures the expected prediction risk of the function $(\hf_{0}, \hh)$ on a new datapoint for the $0$th task, relative to the best prediction rule from which the data was generated---$f^\star_{0} \circ \h^\star$. 

\subsection{Model complexity}
A well-known measure for the complexity of a function class is its Gaussian complexity. For a generic vector-valued function class $\cQ$ containing functions $\q(\cdot) : \mR^d \to \mR^r$, and $N$ data points, $\bX = (\x_1, \ldots, \x_N)^\top$, the empirical Gaussian complexity is defined as
\begin{equation*}
\EMPGaus{\bX}{\cQ} = \mE_{\g}[\sup_{\q \in \cQ} \frac{1}{N} \sum_{k=1}^r \sum_{i=1}^N g_{ki} q_k(\x_{i})], \qquad  g_{ki} \sim \cN(0, 1) \ i.i.d.,
\end{equation*}
where $\g = \{g_{ki}\}_{k\in[r], i\in[N]}$, and $q_k(\cdot)$ is the $k$-th coordinate of the vector-valued function $\q(\cdot)$. We define the corresponding population Gaussian complexity as $\POPGaus{N}{\cQ} = \mE_{\bX} [\EMPGaus{\bX}{\cQ}]$, where the expectation is taken over the distribution of data samples $\bX$. Intuitively, $\POPGaus{N}{\cQ}$ measures the complexity of $\cQ$ by the extent to which functions in the class $\cQ$ can correlate with random noise $g_{ki}$.
\section{Main Results}
\label{sec:transfer}
We now present our central theoretical results for the transfer learning problem. We first present statistical guarantees for the training phase and test phase separately. Then, we present a problem-agnostic definition of task diversity, followed by our generic end-to-end transfer learning guarantee. Throughout this section, we make the following standard, mild regularity assumptions on the loss function $\ell(\cdot, \cdot)$, the function class of tasks $\cF$, and the function class of shared representations $\cH$.
\begin{assumption}[Regularity conditions]
\label{assump:regularity}
The following regularity conditions hold:
\begin{itemize}
\item The loss function $\ell(\cdot, \cdot)$ is $B$-bounded, and $\ell(\cdot, y)$ is $L$-Lipschitz for all $y \in \cY$.
\item The function $f$ is $L(\cF)$-Lipschitz with respect to the $\ell_2$ distance, for any  $f \in \cF$.
\item The composed function $f\circ\h$ is bounded: $\sup_{\x \in \cX} |f \circ \h(\x)| \le D_{\cX}$, for any $f \in \cF, \h \in \cH$.
\end{itemize}
\end{assumption}

We also make the following realizability assumptions, which state that the true underlying task functions and the true representation are contained in the function classes $\cF, \cH$ over which the two-stage ERM oracle optimizes in \cref{eq:metatrainerm} and \cref{eq:metatesterm}.
\begin{assumption}[Realizability] \label{assump:realizable}
The true representation $\h^\star$ is contained in $\cH$. Additionally, the true task specific functions $f^\star_j$ are contained in $\cF$ for both the training tasks and new test task (i.e., for any $j \in [t] \cup \{0\}$).
\end{assumption}

\subsection{Learning shared representations}
In order to measure ``closeness" between the learned representation and true underlying feature representation, we need to define an appropriate distance measure between arbitrary representations. To this end, we begin by introducing the \emph{task-averaged representation difference}, which captures the extent two representations $\h$ and $\h'$ differ in aggregate over the $t$ training tasks measured by the population train loss.

\begin{definition}\label{def:averagecase}
For a function class $\cF$, $t$ functions $\nbf = (f_1, \hdots, f_t)$, and data $(\x_j, y_j) \sim 
\Pr_{f_j \circ \h}$ as in \cref{eq:gen_model} for any $j \in [t]$, the \textbf{task-averaged representation difference} between representations $\h, \h'\in \cH$ is:
\begin{equation*}
\avdist{\f}{\h'}{\h} = \frac{1}{t}\sum_{j=1}^t \inf_{f' \in \cF} \mE_{\x_j, y_j} \Big \{ \ell(f' \circ \h'(\x_j), y_j) - \ell(f_j \circ \h(\x_j), y_j) \Big \}.
\end{equation*}
\end{definition}
Under this metric, we can show that the distance between a learned representation and the true underlying representation is controlled in the training phase. Our following guarantees also feature the \emph{worst-case Gaussian complexity} over the function class $\cF$, which is defined as:\footnote{Note that a stronger version of our results hold with a sharper, data-dependent version of the worst-case Gaussian complexity that eschews the absolute maxima over $\x_i$. See \cref{cor:metatrain_emp} and \cref{thm:chain_rule} for the formal statements.}

\begin{equation} \label{eq:worstcase_Gaussian}
\MAXGaus{n}{\cF} = \max_{\Z \in \cZ}\EMPGaus{\Z}{\cF}, \text{~where~}\cZ = \{ (\h(\x_1), \cdots, \h(\x_n)) ~|~ \h \in \cH, \x_i \in \cX \text{~for all~} i\in[n]\}.
\end{equation}
where $\cZ$ is the domain induced by any set of $n$ samples in $\cX$ and any representation $\h \in \cH$. Moreover, we will always use the subscript $nt$, on $\POPGaus{nt}{\cQ} = \mE_{\X}[\EMPGaus{\X}{\cQ}]$, to refer to the population Gaussian complexity computed with respect to the data matrix $\X$ formed from the concatentation of the $nt$ training datapoints $\{\x_{ji} \}_{j=1, i=1}^{t, n}$. 
We can now present our training phase guarantee.
\begin{theorem}	\label{thm:metatrain}
	Let $\hh$ be an empirical risk minimizer of $\hLtrain(\cdot, \cdot)$ in \cref{eq:metatrainerm}. Then, if  \cref{assump:regularity,,assump:realizable} hold, with probability at least $1-\delta$: 
	\begin{align*} 	
		\avdist{\fstar}{\hh}{\hstar}
		& \leq 16 L \POPGaus{nt}{\cF^{\otimes t} \circ \cH} + 8 B \sqrt{\frac{\log(2/\delta)}{nt}} \\
		& \leq 4096 L \left[ \frac{D_{\cX}}{(nt)^2} + \log(nt) \cdot [L(\cF) \cdot \POPGaus{nt}{\cH} + \MAXGaus{n}{\cF}]\right] + 8 B \sqrt{\frac{\log(2/\delta)}{nt}}.
	\end{align*}
\end{theorem}

 \cref{thm:metatrain} asserts that the \emph{task-averaged representation difference} (\cref{def:averagecase}) between our learned representation and the true representation is upper bounded by the population Gaussian complexity of the vector-valued function class $\cF^{\otimes t} \circ \cH  = \{ (f_1 \circ \h, \hdots, f_t \circ \h) : (f_1, \hdots, f_t) \in \cF^{\otimes t}, \h \in \cH \}$, plus a lower-order noise term. Up to logarithmic factors and lower-order terms, this Gaussian complexity can be further decomposed into the complexity of learning a representation in $\cH$ with $nt$ samples---$L(\cF) \cdot \POPGaus{nt}{\cH}$---and the complexity of learning a task-specific function in $\cF$ using $n$ samples per task---$\MAXGaus{n}{\cF}$. For the majority of parametric function classes used in machine learning applications, $\POPGaus{nt}{\cH} \sim \sqrt{C(\cH) \slash nt}$ and $\MAXGaus{n}{\cF} \sim \sqrt{C(\cF) \slash n}$, where the function $C(\cdot)$ measures the intrinsic complexity of the function class (e.g., VC dimension, absolute dimension, or parameter norm \citep{wainwright2019high}).

We now make several remarks on this result. First, \cref{thm:metatrain} differs from  standard supervised learning generalization bounds. \cref{thm:metatrain} provides a bound on the  distance between two representations as opposed to the empirical or population training risk, despite the lack of access to a direct signal from the underlying feature representation. Second, the decomposition of $\POPGaus{nt}{\cF^{\otimes t} \circ \cH}$ into the individual Gaussian complexities of $\cH$ and $\cF$, leverages a novel chain rule for Gaussian complexities (see \cref{thm:chain_rule}), which may be of independent interest. This chain rule (\cref{thm:chain_rule}) can be viewed as a generalization of classical Gaussian comparison inequalities and results such as the Ledoux-Talagrand contraction principle \citep{ledoux2013probability}. Further details and comparisons to the literature for this chain rule can be found in \cref{app:chainrule}  (this result also avoids an absolute maxima over $\x_i \in \cX)$.

\subsection{Transferring to new tasks}
In addition to the \emph{task-averaged representation difference}, we also introduce the \emph{worst-case representation difference}, which captures the distance between two representations $\h'$, $\h$ in the context of an arbitrary worst-case task-specific function $f_0 \in \cF_0$.

\begin{definition} \label{def:worstcase}
For function classes $\cF$ and $\cF_0$ such that $f_0 \in \cF_0$, and data $(\x, y) \sim \Pr_{f_0 \circ \h}$ as in \cref{eq:gen_model}, the \textbf{worst-case representation difference} between representations $\h, \h' \in \cH$ is:
\begin{equation*}
\dist{\h'}{\h} =  \sup_{f_0 \in \cF_0} \inf_{f' \in \cF} \mE_{\x, y} \Big \{ \ell(f' \circ \h'(\x), y) - \ell(f_0 \circ \h(\x), y) \Big \}.
\end{equation*}
\end{definition}
For flexibility we allow  $\cF_0$ to be distinct from $\cF$ (although in most cases, we choose $\cF_0 \subset \cF$). The function class $\cF_0$ is the set of new tasks on which we hope to generalize. The generalization guarantee for the test phase ERM estimator follows.

\begin{theorem}\label{thm:metatest} 
	Let $\hf_{0}$ be an empirical risk minimizer of $\hLtest(\cdot, \hh)$ in \cref{eq:metatesterm} for any feature representation $\hh$. Then if \cref{assump:regularity,,assump:realizable} hold, and $f^\star_0 \in \cF_0$ for an unknown class $\cF_0$, with probability at least $1-\delta$:
\begin{equation*}
\Ltest(\hf_{0}, \hh) - \Ltest(f^\star_{0}, \h^\star) \leq \dist{\hh}{\h^\star} + 16 L \cdot \MAXGaus{m}{\cF} + 8 B \sqrt{\frac{\log(2/\delta)}{m}}
	\end{equation*}

\end{theorem}
Here $\MAXGaus{m}{\cF}$ is again the worst-case Gaussian complexity\footnote{As before, a stronger version of this result holds with a sharper data-dependent version of the Gaussian complexity in lieu of $\MAXGaus{m}{\cF}$ (see \cref{cor:metatest_funcbernstein}).} as defined in \eqref{eq:worstcase_Gaussian}.
\cref{thm:metatest} provides an excess risk bound for prediction on a  new task in the test phase with two dominant terms. The first is the worst-case representation difference $\dist{\hh}{\h^\star}$, which accounts for the error of using a biased feature representation $\hh \neq \hstar$ in the test ERM procedure. The second is the difficulty of learning $f^{\star}_0$ with $m$ samples, which is encapsulated in $\MAXGaus{m}{\cF}$. 

\subsection{Task diversity and end-to-end transfer learning guarantees}
We now introduce the key notion of task diversity. Since the learner does not have direct access to a signal from the representation, they can only observe partial information about the representation channeled through the composite functions $\fjstar \circ \hstar$. If a particular direction/component in $\hstar$ is not seen by a corresponding task $\fjstar$ in the training phase, that component of the representation $\hstar$ cannot be distinguished from a corresponding one in a spurious $\h'$.  When this component is needed to predict on a new task corresponding to $f^{\star}_0$ which lies along that particular direction, transfer learning will not be possible.
Accordingly, \cref{def:averagecase} defines a notion of representation distance in terms of information channeled through the training tasks, while \cref{def:worstcase} defines it in terms of an arbitrary new test task. Task diversity essentially encodes the ratio of these two quantities (i.e. how well the training tasks can cover the space captured by the representation $\hstar$ needed to predict on new tasks). 
Intuitively, if all the task-specific functions were quite similar, then we would only expect the training stage to learn about a narrow slice of the representation---making transferring to a generic new task difficult.


\begin{definition}
\label{def:diversity}
For a function class $\cF$, we say $t$ functions $\nbf = (f_1, \hdots, f_t)$ are $(\nu, \epsilon)$-\textbf{diverse} over $\cF_0$ for a representation $\h$, if uniformly for all $\h' \in \cH$,
\begin{equation*}
	\dist{\h'}{\h} \leq \avdist{\f}{\h'}{\h} /\nu + \epsilon.
\end{equation*}
\end{definition}
Up to a small additive error $\epsilon$, diverse tasks ensure that the worst-case representation difference for the function class $\cF_0$ is controlled when the task-averaged representation difference for a sequence of $t$ tasks $\f$ is small. 
Despite the abstraction in this definition of task diversity, it \textit{exactly} recovers the notion of task diversity in \citet{tripuraneni2020provable} and \citet{du2020few}, where it is restricted to the special case of linear functions and quadratic loss. Our general notion allows us to move far beyond the linear-quadratic setting as we show in \cref{sec:diversity} and \cref{sec:index}. 

We now utilize the definition of task diversity to merge our training phase and test phase results into an end-to-end transfer learning guarantee
for generalization to the unseen task $f_{0}^{\star} \circ \hstar$. 

\begin{theorem} \label{cor:ltlguarantee}
	Let $(\cdot, \hh)$ be an empirical risk minimizer of $\hLtrain(\cdot, \cdot)$ in \cref{eq:metatrainerm}, and $\hf_{0}$ be an empirical risk minimizer of $\hLtest(\cdot, \hh)$ in \cref{eq:metatesterm} for the learned feature representation $\hh$. Then if \cref{assump:regularity,,assump:realizable} hold, and the training tasks are $(\nu, \epsilon)$-diverse, with probability at least $1-2\delta$, the transfer learning risk in \cref{eq:excessrisk} is upper-bounded by:
	\begin{align*}
		 O \Big( L  \log (nt) \cdot \Big[\frac{L(\cF) \cdot \POPGaus{nt}{\cH} + \MAXGaus{n}{\cF}}{\nu} \Big]
		 +  L \MAXGaus{m}{\cF} + \frac{L D_{\cX}}{\nu (nt)^2} +   B \Big[\frac{1}{\nu} \cdot \sqrt{\frac{\log(2/\delta)}{nt}}+&\sqrt{\frac{\log(2/\delta)}{m}}\Big] + \epsilon \Big).
	\end{align*}
\end{theorem}
Theorem \ref{cor:ltlguarantee} gives an upper bound on the transfer learning risk. The dominant terms in the bound are the three Gaussian complexity terms. For parametric function classes we expect $\POPGaus{nt}{\cH} \sim \sqrt{C(\cH) \slash (nt)}$ and $\MAXGaus{N}{\cF} \sim \sqrt{C(\cF) \slash N}$, where $C(\cH)$ and $C(\cF)$ capture the dimension-dependent size of the function classes. Therefore, when $L$ and $L(\cF)$ are constants, the leading-order terms for the transfer learning risk scale as $\tilde{O}(\sqrt{(C(\cH) + t \cdot C(\cF))/(nt)} + \sqrt{C(\cF)/m})$. A naive algorithm which simply learns the new task in isolation, ignoring the training tasks, has an excess risk scaling as $\tilde{O}(\sqrt{C(\cF \circ \cH)/m}) \approx \tilde{O}(\sqrt{(C(\cH)+C(\cF))/m})$. Therefore, when $n$ and $t$ are sufficiently large, but $m$ is relatively small (i.e., the setting of few-shot learning), the performance of transfer learning is significantly better than the baseline of learning in isolation.

\section{Applications}
\label{sec:diversity}
We now consider a varied set of applications to instantiate our general transfer learning framework. 
In each application, we first specify the function classes and data distributions we are considering as well as our assumptions. We then state the task diversity and the Gaussian complexities of the function classes, which together furnish the bounds on the \textit{transfer learning risk}--from \cref{eq:excessrisk}--in \cref{cor:ltlguarantee}. 
\subsection{Multitask Logistic Regression}
\label{sec:logistic_reg}
We first instantiate our framework for one of the most frequently used classification methods---logistic regression.
Consider the setting where the task-specific functions are linear maps, and the underlying representation is a projection onto a low-dimensional subspace. Formally, let $d \ge r$, and let the function classes $\cF$ and $\cH$ be:
\begin{align}
\cF =& \{ \ f \ | \ f(\z) = \balpha^\top \z, \ \balpha \in \mR^r, \ \norm{\balpha} \le c_1 \},   \label{eq:logistic_class}\\
\cH =& \{ \ \h \ | \ \h(\x) \ = \ \B^\top \x, \ \B  \in \mR^{d \times r}, \ \B \text{~is a matrix with orthonormal columns}\}. \nonumber 
\end{align}
Here $\cX = \mR^d$, $\cY = \{0, 1\}$, and the measure $\Pr_\x$ is $\mSigma$-sub-gaussian (see \cref{def:sg}) and $D$-bounded (i.e., $\norm{\x} \le D$ with probability one). We let the conditional distribution in \eqref{eq:gen_model} satisfy:
\begin{equation*}
\Pr_{y|\x}(y=1| f \circ \h(\x)) =  \sigma(\balpha\trans \B\trans \x),
\end{equation*}
where $\sigma(\cdot)$ is the sigmoid function with $\sigma(z)=1/(1+\exp(-z))$. We use the logistic loss $\ell(z, y) = -y \log(\sigma(z)) - (1-y) \log(1-\sigma(z))$. The true training tasks take the form $f^\star_j (\z) = (\balphastarj)^\top \z$ for all $j \in [t]$, and we let $\A=(\balpha_1^{\star}, \hdots, \balpha_t^{\star})\trans \in \mR^{t \times r}$. We make the following assumption on the training tasks being ``diverse'' and both the training and new task being normalized.
\begin{assumption}\label{assump:linear_diverse}
 $\sigma_r(\A^\top \A/t) = \tnu > 0$ and $\norm{\balphastarj} \leq O(1)$ for $j \in [t] \cup \{ 0 \}$.
\end{assumption}

In this case where the $\cF$ contains underlying linear task functions $\balpha_{j}^{\star} \in \mR^r$ (as in our examples in Section 4), our task diversity definition reduces to ensuring these task vectors span the entire $r$-dimensional space containing the output of the representation $\h(\cdot) \in \mR^r$. This is quantitatively captured by the conditioning parameter $\tnu=\sigma_r(\A)$ which represents how spread out these vectors are in $\mR^r$. The training tasks will be  well-conditioned in the sense that $\sigma_1(\A^\top \A/t)/\sigma_r(\A^\top \A/t) \leq O(1)$ (w.h.p.) for example, if each $\balpha_t \sim \cN(0, \frac{1}{\sqrt{r}} \mathbf{\Sigma})$ i.i.d. where $\sigma_{1}(\mathbf{\Sigma})/\sigma_{r}(\mathbf{\Sigma}) \leq O(1)$.

\cref{assump:linear_diverse} with natural choices of $\cF_0$ and $\cF$ establishes $(\Omega(\tnu), 0)$-diversity as defined in Definition \ref{def:diversity} (see \cref{lem:logistic_diversity}).
Finally, by standard arguments, we can bound the Gaussian complexity of $\cH$ in this setting by $\POPGaus{N}{\cH} \leq \tlO (\sqrt{dr^2/N})$. We can also show that a finer notion of the Gaussian complexity for $\cF$, serving as the analog of $\MAXGaus{N}{\cF}$, is 
upper bounded by $\tlO(\sqrt{r/N})$. This is used to sharply bound the complexity of learning $\cF$ in the training and test phases (see proof of \cref{thm:logistic_transfer} for more details). 
Together, these give the following guarantee.
\begin{theorem}
\label{thm:logistic_transfer}
	If \cref{assump:linear_diverse} holds, $\h^\star(\cdot) \in \cH$, and $\cF_0 = \{ \ f \ | \ f(\x) = \balpha^\top \z, \ \balpha \in \mR^r, \ \norm{\balpha} \le c_2 \}$, then there exist constants $c_1, c_2$ such that the training tasks $f_j^{\star}$ are $(\Omega(\tnu), 0)$-diverse over $\cF_0$. Furthermore, if for a sufficiently large constant $c_3$, $n \ge c_3(d+\log t)$, $m \ge c_3r$, and $D \le c_3(\min(\sqrt{dr^2}, \sqrt{rm}))$, then with probability at least $1-2\delta$:
	\begin{align*}
	\text{Transfer Learning Risk} \le \tlO \left( \frac{1}{\tnu} \lprp{\sqrt{\frac{dr^2}{nt}} +  \sqrt{\frac{r}{n}}} + \sqrt{\frac{r}{m}} \right).
    \end{align*}    
\end{theorem}

A naive bound for logistic regression ignoring the training task data would have a guarantee $O (\sqrt{d \slash m})$. For $n$ and $t$ sufficiently large, the bound in \cref{thm:logistic_transfer} scales as $\tlO(\sqrt{r \slash m})$, which is a significant improvement over $O(\sqrt{d \slash m})$ when $r \ll d$. Note that our result in fact holds with the empirical data-dependent quantities $\tr(\mSigma_{\X})$ and $\sum_{i=1}^r \sigma_i(\mSigma_{\X_j})$ which can be much smaller then their counterparts $d, r$ in \cref{thm:logistic_transfer}, if the data lies on/or close to a low-dimensional subspace\footnote{Here $\mSigma_{\bX}$ denotes the empirical covariance of the data matrix $\bX$. See \cref{cor:logistic_transfer_datadependent} for the formal statement of this sharper, more general result.}. 
\subsection{Multitask Deep Neural Network Regression}
We now consider the setting of real-valued neural network regression. Here the task-specific functions are linear maps as before, but the underlying representation is specified by a depth-$K$ vector-valued neural network:
\begin{align}
	\h(\x) = \W_K \sigma_{K-1} (\W_{K-1}(\sigma_{K-2} (\hdots \sigma(\W_1 \x))))	\label{eq:nn_feature}.
\end{align}
Each $\W_k$ is a parameter matrix, and each $\sigma_k$ is a $\tanh$ activation function. We let $\norm{\W}_{1, \infty} = \max_{j} (\sum_{k} \abs{\W_{j,k}})$ and $\norm{\W}_{\infty \to 2}$ be the induced $\infty$-to-$2$ operator norm. Formally, $\cF$ and $\cH$ are\footnote{For the following we make the standard assumption each parameter matrix $\W_k$ satisfies $\norm{\W_k}_{1, \infty} \leq M(k)$ for each $j$ in the depth-$K$ network
\citep{golowich2017size}, and that the feature map is well-conditioned.
}
\begin{align}
\cF =& \{ \ f \ | \ f(\z) = \balpha^\top \z, \ \balpha \in \mR^r, \ \norm{\balpha} \le c_1 M(K)^2 \},   \label{eq:NN_class}\\
\cH =& \{  \h(\cdot) \in \mR^r \text{ in \cref{eq:nn_feature} for } \W_k : \norm{\W_k}_{1, \infty} \leq M(k) \text{ for } k \in [K-1], \nonumber \\ 
& \max(\norm{\W_K}_{1, \infty}, \norm{\W_K}_{\infty \to 2}) \leq M(K) \text{, such that } \sigma_r \left(\mE_{\x}[\h(\x)\h(\x)^\top] \right) > \Omega(1) \}. \nonumber
\end{align}
We consider the setting where $\cX = \mR^d$, $\cY = \mR$, and the measure $\Pr_\x$ is $D$-bounded. We also let the conditional distribution in \eqref{eq:gen_model} be induced by:
\begin{equation}
y= \balpha^{\top} \h(\x) + \eta \ \text { for } \ \balpha, \h \ \text{ as in } \cref{eq:NN_class},
\end{equation}
with additive noise $\eta$ bounded almost surely by $O(1)$ and independent of $\x$. We use the standard squared loss $\ell(\balpha^\top \h(\x), y) = (y-\balpha^\top \h(\x))^2$, and let the true training tasks take the form $f^\star_j (\z) = (\balphastarj)^\top \z$ for all $j \in [t]$, and set $\A=(\balpha_1^{\star}, \hdots, \balpha_t^{\star})\trans \in \mR^{t \times r}$ as in the previous example. Here we use exactly the same diversity/normalization assumption on the task-specific maps---\cref{assump:linear_diverse}---as in our logistic regression example.

Choosing $\cF_0$ and $\cF$ appropriately establishes a $(\Omega(\tnu), 0)$-diversity as defined in Definition \ref{def:diversity} (see \cref{lem:square_lin_diversity}). Standard arguments as well as results in \citet{golowich2017size} allow us to bound the Gaussian complexity terms as follows (see the proof of \cref{thm:nn_regression} for details):
\begin{align*}
	\POPGaus{N}{\cH} \leq \tlO \lprp{\frac{rM(K) \cdot D \sqrt{K} \cdot \Pi_{k=1}^{K-1} M(k)}{\sqrt{N}}}; \quad \MAXGaus{N}{\cF} \leq \tlO \lprp{\frac{M(K)^3}{\sqrt{N}}}.
\end{align*}
Combining these results yields the following end-to-end transfer learning guarantee.
\begin{theorem}
If \cref{assump:linear_diverse} holds, $\hstar(\cdot) \in \cH$, and $\cF_0 = \{ \ f \ | \ f(\z) = \balpha^\top \z, \ \balpha \in \mR^r, \ \norm{\balpha} \le c_2 \}$, then there exist constants $c_1, c_2$ such that the training tasks $f_j^{\star}$ are $(\Omega(\tnu), 0)$-diverse over $\cF_0$. Further, if $M(K) \geq c_3$ for a universal constant $c_3$, then with probability at least $1-2\delta$:
	\begin{align*}
		\text{Transfer Learning Risk} \le  \tlO \left( \frac{rM(K)^6 \cdot D \sqrt{K} \cdot \Pi_{k=1}^{K-1} M(k)}{\tnu \sqrt{nt}}  + \frac{M(K)^6}{\tnu \sqrt{n}} +  \frac{M(K)^6}{\sqrt{m}} \right).
	\end{align*}
\label{thm:nn_regression}
\end{theorem}
 
The $\poly(M(K))$ dependence of the guarantee on the final-layer weights can likely be improved, but is dominated by the overhead of learning the complex feature map $\hstar(\cdot)$ which has complexity $\poly(M(K)) \cdot D \sqrt{K} \cdot \Pi_{k=1}^{K-1} M(k)$. By contrast a naive algorithm which does not leverage the training samples would have a sample complexity of $\tlO \lprp{\poly(M(K)) \cdot D \sqrt{K} \cdot \Pi_{k=1}^{K-1} M(k)/\sqrt{m}}$ via a similar analysis. Such a rate can be much larger than the bound in \cref{thm:nn_regression} when $nt \gg m$: exactly the setting relevant to that of few-shot learning for which ImageNet pretraining is often used.
\subsection{Multitask Index Models}
\label{sec:index}
To illustrate the flexibility of our framework, in our final example, we consider a classical statistical model: the index model, which is often studied from the perspective of semiparametric estimation \citep{bickel1993efficient}. As flexible tools for general-purpose, non-linear dimensionality reduction, index models have found broad applications in economics, finance, biology and the social sciences  \citep{bickel1993efficient, li2007nonparametric, otwinowski2018inferring}. This class of models has a different flavor then previously considered:  the task-specific functions are nonparametric ``link" functions, while the underlying representation is a one-dimensional projection. Formally, let the function classes $\cF$ and $\cH$ be:
\begin{align}
\cF =& \{ \ f \ | \ f(\z) \text{ is a $1$-Lipschitz, monotonic function bounded in $[0, 1]$} \},   \label{eq:index_class}\\
\cH =& \{ \ \h \ | \ \h(\x) \ = \ \b^\top \x , \ \b \in \mR^{d}, \ \norm{\b} \leq W \}. \nonumber 
\end{align}
We consider the setting where $\cX = \mR^d$, $\cY = \mR$, the measure $\Pr_\x$ is $D$-bounded, and $DW \geq 1$. This matches the setting in \citet{kakade2011efficient}. The conditional distribution in \eqref{eq:gen_model} is induced by:
\begin{equation*}
y = f(\b^\top \x) + \eta \ \text{ for } f, \b \text{ as in } \cref{eq:nn_feature},
	\label{eq:single_index}
\end{equation*}
with additive noise $\eta$ bounded almost surely by $O(1)$ and independent of $\x$. We use the robust $\ell_1$ loss, $\ell(f(\b^\top \x), y) = \abs{y-f(\b^\top \x)}$, in this example. Now, define $\cF_t = \text{conv} \{ f^{\star}_{1}, \hdots, f^{\star}_{t} \}$ as the convex hull of the training task-specific functions $f^\star_j$. 
Given this, we define the $\tepsilon$-enlargement of $\cF_{t}$ by $\cF_{t, \tepsilon} = \{ f : \exists \tf \in \cF_{t } \text{ such that } \sup_z \abs{f(z)-\tf(z)} \leq \tepsilon \}$. 

We prove a transfer generalization bound for $\cF_0 = \cF_{t, \tepsilon}$, for which we can establish $(\tnu, \tepsilon)$-diversity with $\tnu \geq \frac{1}{t}$ as defined in \cref{def:diversity} (see \cref{lem:regression_general_div}). Standard arguments once again show that $\POPGaus{N}{\cH} \leq O \lprp{\sqrt{(W^2 \mE_{\X}[\tr(\mSigma_{\X}])/N)}}$ and $\MAXGaus{N}{\cF} \leq O \lprp{\sqrt{WD/N}}$ (see the proof of \cref{thm:lipschitz} for details). Together these give the following guarantee.
\begin{theorem}
If $\fjstar \in \cF$ for $j \in [t]$, $\hstar(\cdot) \in \cH$, and $f_{0}^\star \in \cF_0 = \cF_{t, \tepsilon}$, then the training tasks are $(\tnu, \tepsilon)$-diverse over $\cF_0$ where $\tnu \geq \frac{1}{t}$. Further, with probability at least $1-2\delta$: 
\begin{align*}
    \text{Transfer Learning Risk}	\leq \tlO \left(\frac{1}{\tnu} \cdot \lprp{\sqrt{\frac{W^2 \mE_{\X}[\tr(\mSigma_{\X})]}{nt}} + \sqrt{\frac{WD}{n}}} + \sqrt{\frac{WD}{m}} \right) + \tepsilon.
\end{align*}
	\label{thm:lipschitz}
\end{theorem}
As before, the complexity of learning the feature representation decays as $n \to \infty$ . Hence if $\mE[\tr(\mSigma_{\X})]$ is large, the aforementioned bound will provide significant savings over the bound which ignores the training phase samples of $O \lprp{\sqrt{(W^2 \mE_{\X}[\tr(\mSigma_{\X})])/m}} + O(\sqrt{WD/m})$. In this example, the problem-dependent parameter $\tnu$ does not have a simple linear-algebraic interpretation. Indeed, in the worst-case it may seem the aforementioned bound degrades with $t$\footnote{Note as $\tnu$ is problem-dependent, for a given underlying  $\fstar$, $\hstar$, $\cF_0$ problem instance, $\tnu$ may be significantly greater than $\frac{1}{t}$. See the proof of \cref{lem:regression_general_div} for details.}. However, note that $\cF_0 = \cF_{t, \tepsilon}$, so those unseen tasks which we hope to transfer to itself \textit{grows} with $t$ unlike in the previous examples. The difficulty of the transfer learning problem also increases as $t$ increases. Finally, this example utilizes the full power of $(\nu, \epsilon)$-diversity by permitting robust generalization to tasks outside $\cF_{t}$, at the cost of a bias term $\tepsilon$ in the generalization guarantee.

\section{Conclusion}

We present a framework for understanding the generalization abilities of generic models which share a common, underlying representation. In particular, our framework introduces a novel notion of task diversity through which we provide guarantees of a fast convergence rate, decaying with \textit{all of the samples} for the transfer learning problem. One interesting direction for future consideration is investigating the effects of relaxing the common design and realizability assumptions on the results presented here. We also believe extending the results herein to accommodate ``fine-tuning" of learned representations -- that is, mildly adapting the learned representation extracted from training tasks to new, related tasks -- is an important direction for future work.

\section{Acknowledgements}
The authors thank Yeshwanth Cherapanamjeri for useful discussions. NT thanks the RISELab at U.C. Berkeley for support. In addition, this work was supported by the Army Research Office (ARO) under contract W911NF-17-1-0304 as part of the collaboration between US DOD, UK MOD and UK Engineering and Physical Research Council (EPSRC) under the Multidisciplinary University Research Initiative (MURI). 

\bibliography{ref}

\begin{thebibliography}{37}
\providecommand{\natexlab}[1]{#1}
\providecommand{\url}[1]{\texttt{#1}}
\expandafter\ifx\csname urlstyle\endcsname\relax
  \providecommand{\doi}[1]{doi: #1}\else
  \providecommand{\doi}{doi: \begingroup \urlstyle{rm}\Url}\fi

\bibitem[Baevski et~al.(2019)Baevski, Edunov, Liu, Zettlemoyer, and
  Auli]{baevski2019cloze}
Alexei Baevski, Sergey Edunov, Yinhan Liu, Luke Zettlemoyer, and Michael Auli.
\newblock Cloze-driven pretraining of self-attention networks.
\newblock \emph{arXiv preprint arXiv:1903.07785}, 2019.

\bibitem[Bartlett et~al.(2005)Bartlett, Bousquet, Mendelson,
  et~al.]{bartlett2005local}
Peter~L Bartlett, Olivier Bousquet, Shahar Mendelson, et~al.
\newblock Local rademacher complexities.
\newblock \emph{The Annals of Statistics}, 33\penalty0 (4):\penalty0
  1497--1537, 2005.

\bibitem[Baxter(2000)]{baxter2000model}
Jonathan Baxter.
\newblock A model of inductive bias learning.
\newblock \emph{Journal of artificial intelligence research}, 12:\penalty0
  149--198, 2000.

\bibitem[Ben-David and Borbely(2008)]{ben2008notion}
Shai Ben-David and Reba~Schuller Borbely.
\newblock A notion of task relatedness yielding provable multiple-task learning
  guarantees.
\newblock \emph{Machine learning}, 73\penalty0 (3):\penalty0 273--287, 2008.

\bibitem[Bengio et~al.(2013)Bengio, Courville, and
  Vincent]{bengio2013representation}
Yoshua Bengio, Aaron Courville, and Pascal Vincent.
\newblock Representation learning: A review and new perspectives.
\newblock \emph{IEEE transactions on pattern analysis and machine
  intelligence}, 35\penalty0 (8):\penalty0 1798--1828, 2013.

\bibitem[Bickel et~al.(1993)Bickel, Klaassen, Bickel, Ritov, Klaassen, Wellner,
  and Ritov]{bickel1993efficient}
Peter~J Bickel, Chris~AJ Klaassen, Peter~J Bickel, Ya’acov Ritov, J~Klaassen,
  Jon~A Wellner, and YA'Acov Ritov.
\newblock \emph{Efficient and adaptive estimation for semiparametric models},
  volume~4.
\newblock Johns Hopkins University Press Baltimore, 1993.

\bibitem[Boyd and Vandenberghe(2004)]{boyd2004convex}
Stephen Boyd and Lieven Vandenberghe.
\newblock \emph{Convex optimization}.
\newblock Cambridge university press, 2004.

\bibitem[Caruana(1997)]{caruana1997multitask}
Rich Caruana.
\newblock Multitask learning.
\newblock \emph{Machine learning}, 28\penalty0 (1):\penalty0 41--75, 1997.

\bibitem[Cavallanti et~al.(2010)Cavallanti, Cesa-Bianchi, and
  Gentile]{cavallanti2010linear}
Giovanni Cavallanti, Nicolo Cesa-Bianchi, and Claudio Gentile.
\newblock Linear algorithms for online multitask classification.
\newblock \emph{Journal of Machine Learning Research}, 11\penalty0
  (Oct):\penalty0 2901--2934, 2010.

\bibitem[Denevi et~al.(2019{\natexlab{a}})Denevi, Ciliberto, Grazzi, and
  Pontil]{denevi2019learning}
Giulia Denevi, Carlo Ciliberto, Riccardo Grazzi, and Massimiliano Pontil.
\newblock Learning-to-learn stochastic gradient descent with biased
  regularization.
\newblock \emph{arXiv preprint arXiv:1903.10399}, 2019{\natexlab{a}}.

\bibitem[Denevi et~al.(2019{\natexlab{b}})Denevi, Stamos, Ciliberto, and
  Pontil]{denevi2019online}
Giulia Denevi, Dimitris Stamos, Carlo Ciliberto, and Massimiliano Pontil.
\newblock Online-within-online meta-learning.
\newblock In \emph{Advances in Neural Information Processing Systems}, pages
  13089--13099, 2019{\natexlab{b}}.

\bibitem[Donahue et~al.(2014)Donahue, Jia, Vinyals, Hoffman, Zhang, Tzeng, and
  Darrell]{donahue2014decaf}
Jeff Donahue, Yangqing Jia, Oriol Vinyals, Judy Hoffman, Ning Zhang, Eric
  Tzeng, and Trevor Darrell.
\newblock Decaf: A deep convolutional activation feature for generic visual
  recognition.
\newblock In \emph{International conference on machine learning}, pages
  647--655, 2014.

\bibitem[Du et~al.(2020)Du, Hu, Kakade, Lee, and Lei]{du2020few}
Simon~S Du, Wei Hu, Sham~M Kakade, Jason~D Lee, and Qi~Lei.
\newblock Few-shot learning via learning the representation, provably.
\newblock \emph{arXiv preprint arXiv:2002.09434}, 2020.

\bibitem[Elnaggar et~al.(2020)Elnaggar, Heinzinger, Dallago, and
  Rost]{Elnaggar864405}
Ahmed Elnaggar, Michael Heinzinger, Christian Dallago, and Burkhard Rost.
\newblock End-to-end multitask learning, from protein language to protein
  features without alignments.
\newblock \emph{bioRxiv}, 2020.
\newblock \doi{10.1101/864405}.

\bibitem[Finn et~al.(2017)Finn, Abbeel, and Levine]{finn2017model}
Chelsea Finn, Pieter Abbeel, and Sergey Levine.
\newblock Model-agnostic meta-learning for fast adaptation of deep networks.
\newblock In \emph{Proceedings of the 34th International Conference on Machine
  Learning-Volume 70}, pages 1126--1135. JMLR. org, 2017.

\bibitem[Finn et~al.(2019)Finn, Rajeswaran, Kakade, and Levine]{finn2019online}
Chelsea Finn, Aravind Rajeswaran, Sham Kakade, and Sergey Levine.
\newblock Online meta-learning.
\newblock \emph{arXiv preprint arXiv:1902.08438}, 2019.

\bibitem[Golowich et~al.(2017)Golowich, Rakhlin, and Shamir]{golowich2017size}
Noah Golowich, Alexander Rakhlin, and Ohad Shamir.
\newblock Size-independent sample complexity of neural networks.
\newblock \emph{arXiv preprint arXiv:1712.06541}, 2017.

\bibitem[Gulshan et~al.(2016)Gulshan, Peng, Coram, Stumpe, Wu, Narayanaswamy,
  Venugopalan, Widner, Madams, Cuadros, et~al.]{gulshan2016development}
Varun Gulshan, Lily Peng, Marc Coram, Martin~C Stumpe, Derek Wu, Arunachalam
  Narayanaswamy, Subhashini Venugopalan, Kasumi Widner, Tom Madams, Jorge
  Cuadros, et~al.
\newblock Development and validation of a deep learning algorithm for detection
  of diabetic retinopathy in retinal fundus photographs.
\newblock \emph{Jama}, 316\penalty0 (22):\penalty0 2402--2410, 2016.

\bibitem[Hospedales et~al.(2020)Hospedales, Antoniou, Micaelli, and
  Storkey]{hospedales2020meta}
Timothy Hospedales, Antreas Antoniou, Paul Micaelli, and Amos Storkey.
\newblock Meta-learning in neural networks: A survey.
\newblock \emph{arXiv preprint arXiv:2004.05439}, 2020.

\bibitem[Kakade et~al.(2011)Kakade, Kanade, Shamir, and
  Kalai]{kakade2011efficient}
Sham~M Kakade, Varun Kanade, Ohad Shamir, and Adam Kalai.
\newblock Efficient learning of generalized linear and single index models with
  isotonic regression.
\newblock In \emph{Advances in Neural Information Processing Systems}, pages
  927--935, 2011.

\bibitem[Khodak et~al.(2019{\natexlab{a}})Khodak, Balcan, and
  Talwalkar]{khodak2019provable}
Mikhail Khodak, Maria-Florina Balcan, and Ameet Talwalkar.
\newblock Provable guarantees for gradient-based meta-learning.
\newblock \emph{arXiv preprint arXiv:1902.10644}, 2019{\natexlab{a}}.

\bibitem[Khodak et~al.(2019{\natexlab{b}})Khodak, Balcan, and
  Talwalkar]{khodak2019adaptive}
Mikhail Khodak, Maria-Florina~F Balcan, and Ameet~S Talwalkar.
\newblock Adaptive gradient-based meta-learning methods.
\newblock In \emph{Advances in Neural Information Processing Systems}, pages
  5915--5926, 2019{\natexlab{b}}.

\bibitem[Ledoux and Talagrand(2013)]{ledoux2013probability}
Michel Ledoux and Michel Talagrand.
\newblock \emph{Probability in Banach Spaces: isoperimetry and processes}.
\newblock Springer Science \& Business Media, 2013.

\bibitem[Lee et~al.(2019)Lee, Maji, Ravichandran, and Soatto]{lee2019meta}
Kwonjoon Lee, Subhransu Maji, Avinash Ravichandran, and Stefano Soatto.
\newblock Meta-learning with differentiable convex optimization.
\newblock In \emph{Proceedings of the IEEE Conference on Computer Vision and
  Pattern Recognition}, pages 10657--10665, 2019.

\bibitem[Li and Racine(2007)]{li2007nonparametric}
Qi~Li and Jeffrey~Scott Racine.
\newblock \emph{Nonparametric econometrics: theory and practice}.
\newblock Princeton University Press, 2007.

\bibitem[Lounici et~al.(2011)Lounici, Pontil, Van De~Geer, Tsybakov,
  et~al.]{lounici2011oracle}
Karim Lounici, Massimiliano Pontil, Sara Van De~Geer, Alexandre~B Tsybakov,
  et~al.
\newblock Oracle inequalities and optimal inference under group sparsity.
\newblock \emph{The annals of statistics}, 39\penalty0 (4):\penalty0
  2164--2204, 2011.

\bibitem[Massart et~al.(2000)]{massart2000constants}
Pascal Massart et~al.
\newblock About the constants in talagrand's concentration inequalities for
  empirical processes.
\newblock \emph{The Annals of Probability}, 28\penalty0 (2):\penalty0 863--884,
  2000.

\bibitem[Maurer(2016)]{maurer2016chain}
Andreas Maurer.
\newblock A chain rule for the expected suprema of gaussian processes.
\newblock \emph{Theoretical Computer Science}, 650:\penalty0 109--122, 2016.

\bibitem[Maurer et~al.(2016)Maurer, Pontil, and
  Romera-Paredes]{maurer2016benefit}
Andreas Maurer, Massimiliano Pontil, and Bernardino Romera-Paredes.
\newblock The benefit of multitask representation learning.
\newblock \emph{The Journal of Machine Learning Research}, 17\penalty0
  (1):\penalty0 2853--2884, 2016.

\bibitem[Obozinski et~al.(2011)Obozinski, Wainwright, Jordan,
  et~al.]{obozinski2011support}
Guillaume Obozinski, Martin~J Wainwright, Michael~I Jordan, et~al.
\newblock Support union recovery in high-dimensional multivariate regression.
\newblock \emph{The Annals of Statistics}, 39\penalty0 (1):\penalty0 1--47,
  2011.

\bibitem[Otwinowski et~al.(2018)Otwinowski, McCandlish, and
  Plotkin]{otwinowski2018inferring}
Jakub Otwinowski, David~M McCandlish, and Joshua~B Plotkin.
\newblock Inferring the shape of global epistasis.
\newblock \emph{Proceedings of the National Academy of Sciences}, 115\penalty0
  (32):\penalty0 E7550--E7558, 2018.

\bibitem[Pontil and Maurer(2013)]{pontil2013excess}
Massimiliano Pontil and Andreas Maurer.
\newblock Excess risk bounds for multitask learning with trace norm
  regularization.
\newblock In \emph{Conference on Learning Theory}, pages 55--76, 2013.

\bibitem[Raghu et~al.(2019)Raghu, Raghu, Bengio, and Vinyals]{raghu2019rapid}
Aniruddh Raghu, Maithra Raghu, Samy Bengio, and Oriol Vinyals.
\newblock Rapid learning or feature reuse? towards understanding the
  effectiveness of maml.
\newblock \emph{arXiv preprint arXiv:1909.09157}, 2019.

\bibitem[Tripuraneni et~al.(2020)Tripuraneni, Jin, and
  Jordan]{tripuraneni2020provable}
Nilesh Tripuraneni, Chi Jin, and Michael~I Jordan.
\newblock Provable meta-learning of linear representations.
\newblock \emph{arXiv preprint arXiv:2002.11684}, 2020.

\bibitem[Vershynin(2010)]{vershynin2010introduction}
Roman Vershynin.
\newblock Introduction to the non-asymptotic analysis of random matrices.
\newblock \emph{arXiv preprint arXiv:1011.3027}, 2010.

\bibitem[Wainwright(2019)]{wainwright2019high}
Martin~J Wainwright.
\newblock \emph{High-dimensional statistics: A non-asymptotic viewpoint},
  volume~48.
\newblock Cambridge University Press, 2019.

\bibitem[Yosinski et~al.(2014)Yosinski, Clune, Bengio, and
  Lipson]{yosinski2014transferable}
Jason Yosinski, Jeff Clune, Yoshua Bengio, and Hod Lipson.
\newblock How transferable are features in deep neural networks?
\newblock In \emph{Advances in neural information processing systems}, pages
  3320--3328, 2014.

\end{thebibliography}
\bibliographystyle{plainnat}

\onecolumn

\newpage
\appendix
\begin{center}{\LARGE \bf Appendices}\end{center}\vskip12pt

\textbf{Notation:}
Here we introduce several additional pieces of notation we will use throughout. 

We use $\mE_{\x}[\cdot]$ to refer to the expectation operator taken over the randomness in the vector $\x$ sampled from a distribution $\Pr_{\x}$. Throughout we will use $\cF$ to refer exclusively to a scalar-valued function class of tasks and $\cH$ to a vector-valued function class of features. For $\cF$, we use $\cF^{\otimes t}$ to refer its $t$-fold Cartesian product such that $(f_1, \hdots, f_t) \equiv \f \in \cF^{\otimes t}$ for $f_j \in \cF$, $j \in [t]$. We use $f(\h)$ as shorthand for the function composition, $f \circ \h$. Similarly, we define the composed function class $\cF(\cH) = \{ f(\h) : f \in \cF, \h \in \cH \}$ and its vector-valued version $\cF^{\otimes t}(\cH) = \{ (f_1(\h), \hdots, f_t(\h)) : f_j \in \cF$, $j \in [t], \h \in \cH \}$ with this shorthand. We will use $\gtrsim$, $\lesssim$, and $\asymp$ to denote greater than, less than, and equal to up to a universal constant and use $\tlO$ to denote an expression that hides polylogarithmic factors in all problem parameters. 

In the context of the two-stage ERM procedure introduced in \cref{sec:prelim} we let the design matrix and responses $y_{ji}$ for the $j$th task be $\X_j$ and $\y_j$ for $j \in [t] \cup \{0\}$, and the entire design matrix and responses concatenated over all $j \in [t]$ tasks as $\X$ and $\y$ respectively. Given a design matrix $\bX = (\x_1, \hdots, \x_N)^\top$ (comprised of mean-zero random vectors) we will let $\mSigma_{\bX} = \frac{1}{N} \bX^\top \bX$ denote its corresponding empirical covariance.

Recall we define the  notions of the empirical and population Gaussian complexity for a generic vector-valued function class $\cQ$ containing functions $\q(\cdot) : \mR^d \to \mR^r$, and data matrix $\X$ with $N$ datapoints as,
\begin{equation*}
\EMPGaus{\X}{\cQ} = \mE_{\g}[\sup_{\q \in \cQ} \frac{1}{N} \sum_{k=1}^r \sum_{i=1}^N g_{ki} q_k(\x_{i})], \qquad \POPGaus{N}{\cQ} = \mE_{\X} [\EMPGaus{\X}{\cQ}]  \qquad  g_{ki} \sim \cN(0, 1) \ i.i.d.,
\end{equation*}
where for the latter population Gaussian complexity each its $N$ datapoints are drawn from the $\Pr_{\x}(\cdot)$ design distribution. Analogously to the above we can define the empirical and population Rademacher complexities for generic vector-valued functions as,
\begin{equation*}
\EMPRad{\X}{\cQ} = \mE_{\bepsilon}[\sup_{\q \in \cQ} \frac{1}{N} \sum_{k=1}^r \sum_{i=1}^N \epsilon_{ki} q_k(\x_{i})], \qquad \POPRad{N}{\cQ} = \mE_{\X} [\EMPRad{\X}{\cQ}]  \qquad  \epsilon_{ki} \sim \text{Rad}(\frac{1}{2}) \ i.i.d.
\end{equation*}

\section{Proofs in \cref{sec:transfer}}
\label{app:transfer}
Here we include the proofs of central generalization guarantees and the Gaussian process chain rule used in its proof.
\subsection{Training Phase/Test Phase Proofs}
\label{app:metatrainmetatest}

In all the following definitions $(\x_j, y_j)$ refer to datapoint drawn from the $j$th component of the model in \cref{eq:gen_model}. We first include the proof of \cref{thm:metatrain} which shows that minimizing the training phase ERM objective controls the task-average distance between the underlying feature representation $\h$ and learned feature representation $\hh$.
\begin{proof}[Proof of \cref{thm:metatrain}]

For fixed $\f', \h'$, define the centered training risk as, 
\begin{align*}
	L(\f', \h', \fstar, \hstar) = \frac{1}{t}\sum_{j=1}^t  \mE_{\x_j, y_j} \Big \{ \ell(f_j' \circ \h'(\x_j), y_j) - \ell(f_j^{\star} \circ \hstar(\x_j), y_j) \Big \}.
\end{align*}
and its empirical counterpart,
\begin{align*}
	\hat{L}(\f', \h', \fstar, \hstar) = \frac{1}{t}\sum_{j=1}^t \sum_{i=1}^{n} \Big \{ \ell(f_j' \circ \h'(\x_{ji}), y_{ji}) - \mE_{\x, y}[\ell(f_j^{\star} \circ \hstar(\x), y)] \Big \}
\end{align*}

Now if $\tbf$ denotes a minimizer of the former expression for fixed $\hh$, in the sense that $\tbf = \frac{1}{t}\sum_{j=1}^t \arg \inf_{f'_j \in \cF} \mE_{\x_j, y_j} \Big \{ \ell(f'_j \circ \hh(\x_j), y_j) - \ell(f_j^{\star} \circ \hstar(\x_j), y_j) \Big \}
$, then by definition, we have that $\avdist{\fstar}{\hh}{\hstar}$ equals the former expression. We first decompose the average distance using the pair $(\hbf, \hh)$. Recall the pair $(\hbf, \hh)$ refers to the empirical risk minimizer in \cref{eq:metatrainerm}.
\begin{align*}
L(\tbf, \hh, \fstar, \hstar) - L(\fstar, \hstar, \fstar, \hstar) = \underbrace{L(\tbf, \hh, \fstar, \hstar) - L(\hbf, \hh, \fstar, \hstar)}_{a} + L(\hbf, \hh, \fstar, \hstar) - L(\fstar, \hstar, \fstar, \hstar)
\end{align*}
Note that by definition of the $\tbf$, $a \leq 0$. The second pair can be controlled via the canonical risk decomposition, 
\begin{align*}
& L(\hbf, \hh, \fstar, \hstar) - L(\fstar, \hstar, \fstar, \hstar)
= \underbrace{L(\hbf, \hh, \fstar, \hstar) - \hat{L}(\hbf, \hh, \fstar, \hstar)}_{b} + \underbrace{ \hat{L}(\hbf, \hh, \fstar, \hstar) - \hat{L}(\fstar, \hstar, \fstar, \hstar)}_{c} + \\
& \underbrace{\hat{L}(\fstar, \hstar, \fstar, \hstar) - L(\fstar, \hstar, \fstar, \hstar)}_{d}.
\end{align*}
By definition $c \leq 0$ (note this inequality uses the realizability in \cref{assump:realizable}) and $b,d \leq \sup_{\f \in \cF^{\otimes t}, \h \in \cH}\abs{\Ltrain(\f, \h)-\hLtrain(\f, \h)}$. By an application of the bounded differences inequality and a standard symmetrization argument (see for example \citet[Theorem 4.10]{wainwright2019high} we have that, 
\begin{align*}
	\sup_{\f \in \cF^{\otimes t}, \h \in \cH}\abs{\Ltrain(\f, \h)-\hLtrain(\f, \h)} \leq 2 \POPRad{nt}{\ell(\cF^{\otimes t}(\cH))} + 2 B \sqrt{\frac{\log(1/\delta)}{nt}}
\end{align*}
with probability at least $1-2\delta$.

It remains to decompose the leading Rademacher complexity term. First we center the functions to $\ell_{ji}(f_j \circ \h(\x_{ji}), y_{ji}) = \ell(f_j \circ \h(\x_{ji}), y_{ji}) - \ell(0, y_{ji})$. Then noting $\abs{\ell_{ji}(0, y_{ji})} \leq B$, the constant-shift property of Rademacher averages \citet[Exercise 4.7c]{wainwright2019high} gives, 
\begin{align*}
	\mE_{\bepsilon}[\sup_{\f \in \cF^{\otimes t}, \h \in \cH} \frac{1}{nt} \sum_{j=1}^t \sum_{i=1}^n \epsilon_{ij} \ell(f_j \circ \h(\x_{ji}), y_{ji})] \leq \mE_{\bepsilon}[\sup_{\f \in \cF, \h \in \cH} \frac{1}{nt} \sum_{j=1}^t \sum_{i=1}^n \epsilon_{ij} \ell_{ij}(f_j \circ \h(\x_{ji}), y_{ji})] + \frac{B}{\sqrt{nt}}
\end{align*}
Now note each $\ell_{ij}(\cdot, \cdot)$ is $L$-Lipschitz in its first coordinate uniformly for every choice of the second coordinate (and by construction centered in its first coordinate). So, defining the set $S=\{ (f_1 \circ \h(\x_{1i}), \hdots, f_j \circ \h(\x_{ji}), \hdots, f_t \circ \h(\x_{ti}))) : j \in [t], f_j \in \cF, \h \in \cH \} \subseteq \mR^{tn}$, and applying the contraction principle \citet[Theorem 4.12]{ledoux2013probability} over this set shows,
\begin{align}
	\mE_{\bepsilon}[\sup_{\f \in \cF^{\otimes t}, \h \in \cH} \frac{1}{nt} \sum_{j=1}^t \sum_{i=1}^n \epsilon_{ij} \ell_{ij}(f_j \circ \h(\x_{ji}), y_{ji})] \leq 2L \cdot \POPRad{nt}{\cF^{\otimes t}(\cH)}. \label{eq:main_sym_bound}
\end{align}
Combining gives,
\begin{align*}
	\sup_{\f \in \cF^{\otimes t}, \h \in \cH}\abs{\Ltrain(\f, \h)-\hLtrain(\f, \h)} \leq 4L \cdot \POPRad{nt}{\cF^{\otimes t}(\cH)} + \frac{4B \sqrt{\log(1/\delta)}}{\sqrt{nt}}
\end{align*}

with probability $1-2\delta$. Now note by \cite[p.97]{ledoux2013probability} empirical Rademacher complexity is upper bounded by empirical Gaussian complexity: $\EMPRad{\X}{\cF^{\otimes t}(\cH)} \leq \sqrt{\frac{\pi}{2}} \EMPGaus{\X}{\cF^{\otimes t}(\cH)}$. Taking expectations of this and combining with the previous display yields the first inequality in the theorem statement.

The last remaining step hinges on \cref{thm:chain_rule} to decompose the Gaussian complexity over $\cF$ and $\cH$. A direct application of \cref{thm:chain_rule} gives the conclusion that,
\begin{align*}
	\EMPGaus{\X}{\cF^{\otimes t}(\cH)} \leq 128 \left (\frac{D_{\X}}{(nt)^2} + C(\cF^{\otimes t}(\cH)) \cdot \log(nt) \right)
\end{align*}
	where $C(\cF^{\otimes t}(\cH); \X) =  L(\cF) \cdot \EMPGaus{\X}{\cH} + \max_{\Z \in \cZ} \EMPGaus{\Z}{\cF}$ where $\cZ = \{ \h(\bX) : \h \in \cH, \bX \in \cup_{j=1}^{t} \{\X_j\} \}$. By definition of $D_{\X}$ we have $D_{\X} \leq 2 D_{\cX}$ and similarly that $\max_{\Z \in \cZ} \EMPGaus{\Z}{\cF} \leq \max_{\Z \in \cZ_1} \EMPGaus{\Z}{\cF}$ for $\cZ_1 = \{ (\h(\x_1), \cdots, \h(\x_n)) ~|~ \h \in \cH, \x_i \in \cX \text{~for all~} i\in[n]\}$. Taking expectations over $\X$ in this series of relations and assembling the previous bounds gives the conclusion after rescaling $\delta$.
\end{proof}

An analogous statement holds both in terms of a sharper notion of the worst-case Gaussian complexity and in terms of empirical Gaussian complexities.

\begin{corollary}
	In the setting of \cref{thm:metatrain},
	\begin{align*}
     \avdist{\fstar}{\hh}{\hstar}		\leq  4096 L \left[ 
	\frac{D_{\cX}}{(nt)^2} + \log(nt) \cdot [L(\cF) \cdot \POPGaus{\X}{\cH} + \mE_{\X}[\max_{\Z \in \cZ} \EMPGaus{\Z}{\cF}]\right] + 8 B \sqrt{\frac{\log(1/\delta)}{n}}
	\end{align*}
	with probability $1-2\delta$ for $\cZ = \{ \h(\bX) : \h \in \cH, \bX \in \cup_{j=1}^{t} \{\X_j\} \}$. Furthermore, 
	\begin{align*}
		& \avdist{\fstar}{\hh}{\hstar}	 \leq 16 \EMPGaus{\X}{\cF^{\otimes t}(\cH)} + 16 B \sqrt{\frac{\log(1/\delta)}{n}}	\leq \\
		& 4096 L \left[ 	\frac{D_{\cX}}{(nt)^2} + \log(nt) \cdot [L(\cF) \cdot \EMPGaus{\X}{\cH} + \max_{\Z \in \cZ} \EMPGaus{\Z}{\cF}]\right] + 16 B \sqrt{\frac{\log(1/\delta)}{n}}	
	\end{align*}
	with probability at least $1-4\delta$. 
	\label{cor:metatrain_emp}
\end{corollary}

\begin{proof}
The argument follows analogously to the proof of \cref{thm:metatrain}. The first statement follows identically by avoiding the relaxation--$\max_{\Z \in \cZ} \EMPGaus{\Z}{\cF} \leq \max_{\Z \in \cZ_1} \EMPGaus{\Z}{\cF}$ for $\cZ_1 = \{ (\h(\x_1), \cdots, \h(\x_n)) ~|~ \h \in \cH, \x_i \in \cX \text{~for all~} i\in[n]\}$--after applying \cref{thm:chain_rule} in the proof of \cref{thm:metatrain}.

The second statement also follows by a direct modification of the proof of \cref{thm:metatrain}. In the proof another application of the bounded differences inequality would show that $\abs{\POPRad{nt}{\cF^{\otimes t}(\cH)}-\EMPRad{\X}{(\cF^{\otimes t}(\cH)}} \leq 4B \sqrt{\frac{\log(1/\delta)}{nt}}$ with probability $1-2\delta$. Applying this inequality after \cref{eq:main_sym_bound} and union bounding over this event and the event in the theorem, followed by the steps in \cref{thm:metatrain}, gives the result after an application of \cref{thm:chain_rule}.
\end{proof}

We now show how the definition of task diversity in \cref{def:diversity} and minimizing the training phase ERM objective allows us to transfer a fixed feature representation $\hh$ and generalize to a new task-specific mapping $f_{0}$. 

\begin{proof}[Proof of \cref{thm:metatest}]
Note $\tf_{0} = \argmin_{f \in \cF} \Ltest(f, \hh)$--it is a minimizer of the population test risk loaded with the fixed feature representation $\hh$. The approach to controlling this term uses the canonical risk decomposition,
\begin{align*}
\Ltest(\hf_{0}, \hh) - \Ltest(\tf_{0}, \hh)
=\underbrace{\Ltest(\hf_{0}, \hh) - \hLtest(\hf_{0}, \hh)}_{a} + \underbrace{\hLtest(\hf_{0}, \hh) - \hLtest(\tf_{0}, \hh)}_{b} + \underbrace{\hLtest(\tf_{0}, \hh) - \Ltest(\tf_{0}, \hh)}_{c}
\end{align*}
First by definition, $b \le 0$.
Now a standard uniform convergence/symmetrization argument which also follows the same steps as in the proof of \cref{thm:metatrain},
\begin{align*}
	a+c \leq 16L \cdot  \mE_{\X_0}[\EMPGaus{\Z_{\hh}}{\cF}] + 8 B \sqrt{\frac{\log(1/\delta)}{m}} \leq 16L \max_{\hh \in \cH} \mE_{\X_0}[\EMPGaus{\Z_{\hh}}{\cF}] + 8B \sqrt{\frac{\log(1/\delta)}{m}}
\end{align*}
for  $\Z_{\hh} = \hh(\X_{0})$, with probability at least $1-2\delta$. The second inequality simply uses the fact that the map $\hh$ is fixed, and independent of the randomness in the test data.
The bias from using an imperfect feature representation $\hh$ in lieu of $\h$ arises in $\Ltest(\tf_{0}, \hh)$. For this term,
\begin{align*}
	& \Ltest(\tf_{0}, \hh) - \Ltest(f_{0}, \hstar) = \inf_{\tf_{0} \in \cF} \{ \Ltest(\tf_{0}, \hh) - \Ltest(f_{0}, \hstar) \} \leq \sup_{f_0 \in \cF_0} \inf_{\tf_{0} \in \cF} \{ L(\tf_{0}, \hh) - L(f_{0}, \hstar) \} = \\
	& \dist{\h}{\hh}
\end{align*}

To obtain the final theorem statement we use an additional relaxation on the Gaussian complexity term for ease of presentation,
\begin{align*}
    \max_{\hh \in \cH} \mE_{\X_0}[\EMPGaus{\Z_{\hh}}{\cF}] \leq \MAXGaus{m}{\cF}.
\end{align*}
Combining terms gives the conclusion. 
\end{proof}

We also present a version of \cref{thm:metatest} which can possess better dependence on the boundedness parameter in the noise terms and has data-dependence in the Gaussian complexities. As before our guarantees can be stated both in terms of population or empirical quantities.
The result appeals to the functional Bernstein inequality instead of the bounded differences inequality in the concentration step. Although we only state (and use) this guarantee for the test phase generalization an analogous statement can be shown to hold for \cref{thm:metatrain}. Throughout the following, we use $(\x_i, y_i) \sim \Pr_{f_0 \circ \h}$ for $i \in [m]$ for ease of notation.

\begin{corollary}	
	In the setting of \cref{thm:metatest}, assuming the loss function $\ell$ satisfies the centering $\ell(0, y)=0$ for all $y \in \cY$,
\begin{align*}
\Ltest(\hf_{0}, \hh) - \Ltest(f^\star_{0}, \h^\star) \leq \dist{\hh}{\h^\star} + 16 L \cdot \mE_{\X_0}[\EMPGaus{\Z_{\hh}}{\cF}] + 4 \sigma \sqrt{\frac{\log(2/\delta)}{m}} + 50 B \frac{\log(2/\delta)}{m}
\end{align*}
for $\Z_{\hh} = \hh(\X_0)$, with probability at least $1-\delta$. Here the maximal variance $\sigma^2 = \frac{1}{m} \sup_{f \in \cF} \sum_{i=1}^m \Var(\ell(f \circ \hh(\x_i), y_i))$. Similarly we have that, 
\begin{align*}
\Ltest(\hf_{0}, \hh) - \Ltest(f^\star_{0}, \h^\star) \leq \dist{\hh}{\h^\star} + 32L \cdot  \EMPGaus{\Z_{\hh}}{\cF} + 8 \sigma \sqrt{\frac{\log(2/\delta)}{m}} + 100 B \frac{\log(2/\delta)}{m}
\end{align*}
with probability at least $1-2\delta$.
	\label{cor:metatest_funcbernstein}	
\end{corollary}
\begin{proof}[Proof of \cref{cor:metatest_funcbernstein}]
	The proof is identical to the proof of \cref{thm:metatest} save in how the concentration argument is performed. Namely in the notation of \cref{thm:metatest}, we upper bound,
	\begin{align*}
		a+c \leq 2 \sup_{f \in \cF} \abs{\hLtest(f, \hh)-\Ltest(f, \hh)} = 2 Z
	\end{align*}
	Note by definition $\mE_{\X_0, \y_0}[\hLtest(f, \hh)] = \Ltest(f, \hh)$, where $\hLtest(f, \hh) = \frac{1}{m} \sum_{i=1}^m \ell(f \circ \hh(\x_i), y_i)$, and the expectation is taken over the test-phase data. Instead of applying the bounded differences inequality to control the fluctuations of this term we apply a powerful form of the functional Bernstein inequality due to \citet{massart2000constants}. Applying \citet[Theorem 3]{massart2000constants} therein, we can conclude,
	\begin{align*}
		Z \leq (1+\epsilon) \mE[Z] + \frac{\sigma}{\sqrt{n}} \sqrt{2 \kappa \log(\frac{1}{\delta})} + \kappa(\epsilon) \frac{B}{m} \log(\frac{1}{\delta})
	\end{align*}
	for $\kappa=2$, $\kappa(\epsilon) = 2.5 + \frac{32}{\epsilon}$ and $\sigma^2 = \frac{1}{m} \sup_{f \in \cF} \sum_{i=1}^m \Var(\ell(f \circ \hh(\x_i), y_i))$. We simply take $\epsilon=1$ for our purposes, which gives the bound,
	\begin{align*}
		Z \leq 2 \mE[Z] + 4 \frac{\sigma}{\sqrt{m}} \sqrt{\log(\frac{1}{\delta})} + 35 \frac{B}{m} \log(\frac{1}{\delta})
	\end{align*}
	
	Next note a standard symmetrization argument shows that $\mE[Z] \leq 2  \mE_{\X_0, \y_0}[\EMPRad{\Z_{\hh}}{\ell \circ \cF}]$ for  $\Z_{\hh} = \hh(\X_{0})$. Following the proof of \cref{thm:metatest} but eschewing the unnecessary centering step in the application of the contraction principle shows that, $\EMPRad{\Z_{\hh}}{\ell \circ \cF} \leq 2 L \cdot \EMPRad{\Z_{\hh}}{\cF}$. Upper bounding empirical Rademacher complexity by Gaussian complexity and following the steps of \cref{thm:metatest} gives the first statement.
	
	The second statement in terms of empirical quantities follows similarly. First the population Rademacher complexity can be converted into an empirical Rademacher complexity using a similar concentration inequality based result which appears in a convenient form in \citet[Lemma A.4 (i)]{bartlett2005local}.  Directly applying this result (with $\alpha=\frac{1}{2}$) shows that, 
	\begin{align*}
		\mE_{\X_0, \y_0}[\EMPRad{\Z_{\hh}}{\ell \circ \cF}] \leq 2 \EMPRad{\Z_{\hh}}{\ell \circ \cF} + \frac{8B \log(\frac{1}{\delta})}{m}
	\end{align*}
	with probability at least $1-\delta$. The remainder of the argument follows exactly as before and as in the proof of \cref{thm:metatest} along with another union bound.
\end{proof}

The proof of \cref{cor:ltlguarantee} is almost immediate.
\begin{proof}[Proof of \cref{cor:ltlguarantee}]
	The result follows immediately by combining \cref{thm:metatrain}, \cref{thm:metatest}, and the definition of task diversity along with a union bound over the two events on which \cref{thm:metatrain,thm:metatest} hold.
\end{proof}

\subsection{A User-Friendly Chain Rule for Gaussian Complexity}
\label{app:chainrule}
We provide the formal statement and the proof of the chain rule for Gaussian complexity that is used in the main text to decouple the complexity of learning the class $\cF^{\otimes t}(\cH)$ into the complexity of learning each individual class. We believe this result may be a technical tool that is of more general interest for a variety of learning problems where compositions of function classes naturally arise.  

Intuitively, the chain rule (\cref{thm:chain_rule}) can be viewed as a generalization of the Ledoux-Talagrand contraction principle which shows that for a \textit{fixed}, centered $L$-Lipschitz function $\phi$, $\EMPGaus{\X}{\phi(\cF)} \le 2 L \EMPGaus{\X}{\cF}$. However, as we are learning \textit{both} $\f \in \cF^{\otimes t}$ (which is not fixed) and $\h \in \cH$, $\EMPGaus{\X}{\cF^{\otimes t} \circ\cH }$ features a suprema over both $\cF^{\otimes t}$ and $\cH$. 

A comparable result for Gaussian processes to our \cref{thm:chain_rule} is used in \cite{maurer2016benefit} for multi-task learning applications, drawing on the chain rule of \citet{maurer2016chain}. Although their result is tighter with respect to logarithmic factors, it cannot be written purely in terms of Gaussian complexities. Rather, it includes a worst-case ``Gaussian-like" average (\citet[Eq. 4]{maurer2016benefit}) in lieu of $\EMPGaus{\cZ}{\cF}$ in \cref{thm:chain_rule}. In general, it is not clear how to sharply bound this term beyond the using existing tools in the learning theory literature. The terms appearing in \cref{thm:chain_rule} can be bounded, in a direct and modular fashion, using the wealth of existing results and tools in the learning theory literature.

Our proof technique and that of \citet{maurer2016chain} both hinge on several properties of Gaussian processes. \citet{maurer2016chain} uses a powerful generalization of the Talagrand majorizing measure theorem to obtain their chain rule. We take a different path. First we use the entropy integral to pass to the space of covering numbers--where  the metric properties of the distance are used to decouple the features and tasks. Finally an appeal to Gaussian process lower bounds are used to come back to expression that involves only Gaussian complexities.

 We will use the machinery of empirical process theory throughout this section so we introduce several useful definitions we will need. We define the empirical $\ell_2$-norm as, $d_{2, \X}^2(\f(\h), \f'(\h')) = \frac{1}{t \cdot n} \sum_{j=1}^t \sum_{i=1}^{n} (f_j(\h(\x_{ji}))-f'_j(\h'(\x_{ji}))^2$, and the corresponding $u$-covering number as $N_{2, \X}(u; d_{2, \X}, \cF^{\otimes t}(\cH))$. Further, we can define the \textit{worst-case} $\ell_2$-covering number as $N_{2}(u; \cF^{\otimes t}(\cH)) = \max_{\X} N_{2, \X}(u; d_{2, \X}, \cF^{\otimes t}(\cH))$. For a vector-valued function class we define the empirical $\ell_2$-norm similarly as $d_{2, \X}^2 (\h, \h') = \frac{1}{t \cdot n} \sum_{k=1}^r \sum_{j=1}^t \sum_{i=1}^{n}  (\h_k(\x_{ji})-\h_k'(\x_{ji}))^2$.

Our goal is to bound the empirical Gaussian complexity of the set $S = \{ (f_1(\h(\x_{1i})), \hdots, f_j(\h(\x_{ji})), \hdots, f_{t}(\h(\x_{ti})))  : j \in [t], f_j \in \cF, \h \in \cH \} \subseteq \mathbb{R}^{tn}$ or function class,
\begin{align*}
	\EMPGaus{nt}{S} = \EMPGaus{\X}{\cF^{\otimes t}(\cH)}= \frac{1}{n t} \mE[ \sup_{\f \in \cF^{\otimes t}, \h \in \cH} \sum_{j=1}^{t} \sum_{i=1}^{n} g_{ji} f_j(\h(\x_{ji})) ] ; \quad g_{ji} \sim \cN(0,1)
\end{align*}
in a manner that allows for easy application in several problems of interest. To be explicit, we also recall that,
\begin{align*}
	\EMPGaus{\X}{\cH} = \frac{1}{nt} \mE_{\g}[\sup_{\h \in \cH} \sum_{k=1}^r \sum_{j=1}^t \sum_{i=1}^n g_{kji} \h_k(\x_{ji})] ; \quad g_{kji} \sim \cN(0,1) 
\end{align*}

We now state the decomposition theorem for Gaussian complexity.
\begin{theorem}
	Let the function class $\cF$ consist of functions that are $\ell_2$-Lipschitz with constant $L(\cF)$, and have boundedness parameter $D_{\X}=\sup_{\f,\f',\h, \h'} d_{2, \X}(\f(\h), \f'(\h'))$. Further, define $\cZ = \{ \h(\bX) : \h \in \cH, \bX \in \cup_{j=1}^{t} \{\X_j\} \}$. Then the (empirical) Gaussian complexity of the function class $\cF^{\otimes t}(\cH)$ satisfies,
	\begin{align*}
		\EMPGaus{\X}{\cF^{\otimes t}(\cH)} \leq \inf_{D_{\X} \geq \delta > 0} \left \{ 4\delta + 64 C(\cF^{\otimes t}(\cH)) \cdot \log \left( \frac{D_{\X}}{\delta} \right) \right \} \leq \frac{4  D_{\X}}{(nt)^2} + 128 C(\cF^{\otimes t}(\cH)) \cdot \log \left( nt \right)
	\end{align*}
	where $C(\cF^{\otimes t}(\cH)) =  L(\cF) \cdot \EMPGaus{\X}{\cH} + \max_{\Z \in \cZ} \EMPGaus{\Z}{\cF}$. Further, if $C(\cF^{\otimes t}(\cH)) \leq D_{\X}$ then by computing the exact infima of the expression,
	\begin{align*}
		\EMPGaus{\X}{\cF^{\otimes t}(\cH)} \leq 64 \left( C(\cF^{\otimes t}(\cH)) + C(\cF^{\otimes t}(\cH)) \cdot \log \left( \frac{D_{\X}}{C(\cF^{\otimes t}(\cH))} \right) \right)
	\end{align*}
	\label{thm:chain_rule}
\end{theorem}

\begin{proof}
For ease of notation we define $N=nt$ in the following.
We can rewrite the Gaussian complexity of the function class $\cF^{\otimes t}(\cH)$ as,
\begin{align*}
	 \EMPGaus{\X}{\cF^{\otimes t}(\cH)} = \mE[\frac{1}{nt} \sup_{\f (\h) \in \cF^{\otimes t}(\cH)} \sum_{j=1}^{t} \sum_{i=1}^{n} g_{ji} f_j(\h(\x_{ji}))] = \mE[\frac{1}{\sqrt{N}} \cdot \sup_{\f (\h) \in \cF^{\otimes t}(\cH)} Z_{\f(\h)}]
\end{align*}
from which we define the mean-zero stochastic process $ Z_{\f(\h)} = \frac{1}{\sqrt{N}} \sum_{j=1}^t \sum_{i=1}^{n} g_{ji} f_j(\h(\x_{ji}))$ for a fixed sequence of design points $\x_{ji}$, indexed by elements $\{ \f(\h) \in \cF^{\otimes t}(\cH) \}$, and for a sequence of independent Gaussian random variables $g_{ji}$. Note the process $Z_{\f(\h)}$ has sub-gaussian increments, in the sense that,
	$Z_{\f(\h)}-Z_{\f'(\h')}$ is a sub-gaussian random variable with parameter $d_{2, \X}^2(\f(\h), \f'(\h')) = \frac{1}{N} \sum_{j=1}^t \sum_{i=1}^{n} (f_j(\h(\x_{ji}))-f'_j(\h'(\x_{ji}))^2$. Since $Z_{\f(\h)}$ is a mean-zero stochastic process we have that, $\mE[\sup_{\f(\h) \in \cF^{\otimes t}(\cH)} Z_{\f(\h)} ] = \mE[\sup_{\f(\h) \in \cF^{\otimes t}(\cH)} Z_{\f(\h)} - Z_{\f'(\h')}] \leq \mE[\sup_{\f(\h), \f'(\h') \in \cF^{\otimes t}(\cH)} Z_{\f(\h)} - Z_{\f'(\h')}]$. Now an appeal to the Dudley entropy integral bound, \citet[Theorem 5.22]{wainwright2019high} shows that,
		\begin{align*} \mE[\sup_{\f(\h), \f(\h') \in \cF^{\otimes t}(\h)} Z_{\f(\h)} - Z_{\f(\h')}] \leq & 4 \mE[\sup_{d_{2, \X}(\f(\h), \f(\h')) \leq \delta} Z_{\f(\h)}-Z_{\f(\h')}] +  32 \int_{\delta}^{D} \sqrt{\log N_{\X}(u; d_{2, \X}, \cF^{\otimes t}(\cH))} du.
		\end{align*}
		We now turn to bounding each of the above terms. Parametrizing the sequence of i.i.d. gaussian variables as $\g$, it follows that $\sup_{d_{2, \X}(\f(\h), \f(\h')) \leq \delta}  Z_{\f(\h)}-Z_{\f(\h')} \leq \sup_{\v : \norm{\v}_2 \leq \delta} \g \cdot \v \leq \norm{\g} \delta$. The corresponding expectation bound, after an application of Jensen's inequality to the $\sqrt{\cdot}$ function gives $\mE[\sup_{d_{2, \X}(\f(\h), \f(\h')) \leq \delta} Z_{\f(\h)}-Z_{\f(\h')}] \leq \mE[\norm{\g}_2 \delta] \leq \sqrt{N} \delta$.
		
		We now turn to bounding the second term by decomposing the distance metric $d_{2, \X}$ into a distance over $\cF^{\otimes t}$ and a distance over $\cH$. We then use a covering argument on each of the spaces $\cF^{\otimes t}$ and $\cH$ to witness a covering of the composed space $\cF^{\otimes t}(\cH)$. Recall we refer to the entire dataset concatenated over the $t$ tasks as $\X \equiv \{ \x_{ji} \}_{j=1, i=1}^{t, n}$. First, let $C_{\cH_{\X}}$ be a covering of the of function space $\cH$ in the empirical $\ell_2$-norm with respect to the inputs $\X$ at scale $\epsilon_1$. Then for each $\h \in C_{\cH_{\X}}$, construct an $\epsilon_2$-covering, $C_{\cF^{\otimes t}_{\h(\X)}}$, of the function space $\cF^{\otimes t}$ in the empirical $\ell_2$-norm with respect to the inputs $\h(\X)$ at scale $\epsilon_2$. We then claim that set $C_{\cF^{\otimes t}(\cH)} = \cup_{\h \in C_ {\cH_{\X}}} ( C_{\cF^{\otimes t}_{\h(\X)}} )$ is an $\epsilon_1 \cdot L(\cF) + \epsilon_2$-cover for the function space $\cF^{\otimes t}(\cH)$ in the empirical $\ell_2$-norm over the inputs $\X$.
		To see this, let $\h \in \cH$ and $\f \in \cF^{\otimes t}$ be arbitrary. Now let $\h' \in C_{\cH_{\X}}$ be $\epsilon_1$-close to $\h$. Given this $\h'$, there exists $\f' \in C_{\cF^{\otimes t}_{\h'(\X)}}$ such that  $\f'$ is $\epsilon_2$-close to $\f$ with respect to inputs $\h'(\X)$. By construction $(\h', \f') \in C_{\cF^{\otimes t}(\cH)}$.
		Finally, using the triangle inequality, we have that,
		\begin{align*}
			& d_{2, \X}(\f(\h), \f'(\h')) \leq d_{2, \X}(\f(\h), \f(\h')) + d_{2, \X}(\f(\h'), \f'(\h')) = \\
			& \sqrt{\frac{1}{N} \sum_{j=1}^t \sum_{i=1}^{n} (f_j(\h(\x_{ji}))-f_j(\h'(\x_{ji})))^2} + 
			\sqrt{\frac{1}{N} \sum_{j=1}^t  \sum_{i=1}^{n} (f_j(\h'(\x_{ji}))-f_j'(\h'(\x_{ji})) )^2} \leq \\
			& L(\cF) \sqrt{\frac{1}{N} \sum_{k=1}^{r} \sum_{j=1}^t \sum_{i=1}^{n} (\h_k(\x_{ji})-\h_k'(\x_{ji}))^2} + 
			\sqrt{\frac{1}{N} \sum_{j=1}^t  \sum_{i=1}^{n} (f_j(\h'(\x_{ji}))-f_j'(\h'(\x_{ji})) )^2} = \\
			& L(\cF) \cdot d_{2, \X}(\h, \h') + d_{2, \h'(\X)}(\f, \f') \leq \epsilon_1 \cdot L(\cF) + \epsilon_2
		\end{align*}
	appealing to the uniform Lipschitz property of the function class $\cF$ in moving from the second to third line, which establishes the claim. 
	
	We now bound the cardinality of the covering 
	$C_{\cF^{\otimes t}(\cH)}$. First, note $\abs{C_{\cF^{\otimes t}(\cH)}} = \sum_{\h \in C_{\cH_{\X}}} \abs{C_{\cF^{\otimes t}_{\h(\X)}}} \leq \abs{C_{\cH_{\X}}} \cdot \max_{\h \in \cH_{\X}} \abs{C_{\cF^{\otimes t}_{\h(\X)}}} $. To control $\max_{\h \in \cH_{\X}} \abs{C_{\cF^{\otimes t}_{\h(\X)}}} $, note an $\epsilon$-cover of $\cF^{\otimes t}_{\h(\X)}$ in the empirical $\ell_2$-norm with respect to $\h(\X)$ can be obtained from the cover $C_{\cF_{\h(\X_1)}} \times \hdots \times C_{\cF_{\h(\X_t)}}$ where $C_{\cF_{\h(\X_i)}}$ denotes a $\epsilon$-cover of $\cF$ in the empirical $\ell_2$-norm with respect to $\h(\X_i)$. Hence $\max_{\h \in \cH_{\X}} \abs{C_{\cF^{\otimes t}_{\h(\X)}}} \leq \abs{C_{\cF_{\h(\X_1)}} \times \hdots \times C_{\cF_{\h(\X_t)}}} \leq \abs{\underbrace{\max_{\z \in \cZ} C_{\cF_{\z}} \times \hdots \times \max_{\z \in \cZ} C_{\cF_{\z}}}_{t \text{ times}}} \leq \abs{\max_{\z \in \cZ} C_{\cF_{\z}}}^t$. Combining these facts provides a bound on the metric entropy of,
	\begin{align*}
		\log N_{2, \X}(\epsilon_1 \cdot L(\cF) + \epsilon_2, d_{2, \X}, \cF^{\otimes t}(\cH)) \leq \log N_{2, \X}(\epsilon_1, d_{2, \X}, \cH) + t \cdot \max_{\Z \in \cZ} \log N_{2, \Z}(\epsilon_2, d_{2, \Z}, \cF)	.
	\end{align*}
Using the covering number upper bound with $\epsilon_1=\frac{\epsilon}{2\cdot L(\cF)}$, $\epsilon_2 = \frac{\epsilon}{2}$ and sub-additivity of the $\sqrt{\cdot}$ function then gives a bound on the entropy integral of,
	\begin{align*}
		\int_{\delta}^{D} \sqrt{\log N_2(\epsilon, d_{2, \X}, \cF^{\otimes t}(\cH))}	 \ d\epsilon \leq \int_{\delta}^D \sqrt{\log N_{2, \X}(\epsilon/(2 L(\cF)), d_{2, \X}, \cH)} \ d \epsilon + \sqrt{t} \int_{\delta}^{D} \max_{\Z \in \cZ} \sqrt{\log N_{2, \Z} (\frac{\epsilon}{2}, d_{2, \Z}, \cF)} \ d \epsilon
	\end{align*}
	From the Sudakov minoration theorem \citet{wainwright2019high}[Theorem 5.30] for Gaussian processes and the fact packing numbers at scale $u$ upper bounds the covering number at scale $u$ we find:
	\begin{align*}
			\log N_{2, \X}(u; d_{2, \X}, \cH) \leq 4 \left (\frac{\sqrt{nt} \EMPGaus{\X}{\cH}}{u} \right)^2 \ \forall  u>0
			\quad \text{ and } \quad \log N_{2, \Z}(u; d_{2, \Z}, \cF) \leq 4 \left( \frac{\sqrt{n} \EMPGaus{\Z}{\cF}}{u} \right)^2 \ \forall  u>0.
	\end{align*}
	For the $\cH$ term  we apply the result to the mean-zero Gaussian process $Z_{\h} = \frac{1}{\sqrt{nt}} \sum_{k=1}^r \sum_{j=1}^t \sum_{i=1}^n g_{kji} h_k(\x_{ji})$, for $g_{kji} \sim \cN(0,1)$ i.i.d. and $\h \in \cH$.
	Combining all of the aforementioned upper bounds, shows that
	\begin{align*}
		& \EMPGaus{\X}{\cF^{\otimes t}(\cH)} \leq \frac{1}{\sqrt{nt}} \left( 4 \delta \sqrt{nt} + 64 L(\cF) \cdot \EMPGaus{\X}{\cH} \cdot \sqrt{nt} \int_{\delta}^{D_{\X}} \frac{1}{u} du  +  64 \sqrt{nt} \cdot \max_{\Z \in \cZ} \EMPGaus{\Z}{\cF} \int_{\delta}^{D_{\X}} \frac{1}{u} du \right) \leq \\
		& 4\delta + 64(L(\cF) \cdot \EMPGaus{\X}{\cH} + \max_{\Z \in \cZ} \EMPGaus{\Z}{\cF}) \cdot \log \left(\frac{D_{\X}}{\delta} \right)  = \delta + C(\cF^{\otimes t}(\cH)) \cdot \log \left(\frac{D_{\X}}{\delta} \right)
		\end{align*}
	defining $C(\cF^{\otimes t}(\cH)) = L(\cF) \cdot \EMPGaus{\X}{\cH} + \max_{\Z \in \cZ} \EMPGaus{\Z}{\cF}$. Choosing $\delta=D_{\X}/(nt)^2$  gives the first inequality. Balancing the first and second term gives the optimal choice $\delta = \frac{1}{C(\cF^{\otimes t}(\cH))}$ for the second inequality under the stated conditions.
\end{proof}


\section{Proofs in \cref{sec:diversity}}
\label{app:diversity}
 
 In this section we instantiate our general framework in several concrete examples. This consists of two steps: first verifying a task diversity lower bound for the function classes and losses and then bounding the various complexity terms appearing in the end-to-end LTL guarantee in \cref{cor:ltlguarantee} or its variants.

\subsection{Logistic Regression}
\label{app:logistic}

Here we include the proofs of the results which both bound the complexities of the function classes $\cF$ and $\cH$ in the logistic regression example as well establish the task diversity lower bound in this setting. In this section we use the following definition,
\begin{definition}
We say the covariate distribution $\Pr_{\x}(\cdot)$ is $\mSigma$-sub-gaussian if for all $\v \in \mR^d$, $\mE[\exp(\v^\top \x_i)] \leq \exp \left( \frac{\Vert \mSigma^{1/2} \v \Vert^2}{2} \right)$ where the covariance $\mSigma$ further satisfies $\sigma_{\max}(\mSigma) \leq C$ and $\sigma_{\min}(\mSigma) \geq c > 0$ for universal constants $c, C$. 
\label{def:sg}
\end{definition}

We begin by presenting the proof of the \cref{thm:logistic_transfer} which essentially relies on instantiating a variant of \cref{cor:ltlguarantee}. In order to obtain a sharper dependence in the noise terms in the test learning stage we actually directly combine \cref{cor:metatrain_emp,cor:metatest_funcbernstein}.

Since we are also interested in stating data-dependent guarantees in this section we use the notation $\mSigma_{\X} = \frac{1}{nt} \sum_{j=1}^t \sum_{i=1}^n \x_{ji} \x_{ji}^\top$ to refer to the empirical covariance across the the training phase samples and $\mSigma_{\X_j}$ for corresponding empirical covariances across the per-task samples. Immediately following this result we present the statement of sharp data-dependent guarantee which depends on these empirical quantities for completeness.

\begin{proof}[Proof of \cref{thm:logistic_transfer}]
First note due to the task normalization conditions we can choose $c_1, c_2$ sufficiently large so that the realizability  assumption in \cref{assump:realizable} is satisfied--in particular, we can assume that $c_2$ is chosen large enough to contain all the parameters $\balphastarj$ for $j \in [t] \cup \{0\}$ and $c_1 \geq \frac{C}{c} c_2$. Next note that under the conditions of the result we can use \cref{lem:logistic_diversity} to verify the task diversity condition is satisfied with parameters $(\tnu, 0)$ with $\nu = \sigma_{r}(\A^\top \A/t) > 0$ with this choice of constants.

Finally, in order to combine \cref{cor:metatrain_emp,cor:metatest_funcbernstein} we begin by bounding each of the complexity terms in the expression. First,
	\begin{itemize}[leftmargin=.5cm]
	\item In the following we use $\b_k$ for $k \in [r]$ to index the orthonormal columns of $\B$. For the feature learning complexity in the training phase we obtain,
	\begin{align*}
		& \EMPGaus{\X}{\cH}=\frac{1}{nt} \mE[ \sup_{\B \in \cH} \sum_{k=1}^r \sum_{j=1}^{t} \sum_{i=1}^{n} g_{kji} \bb_k^\top \x_{ji}] = \frac{1}{nt} \mE[ \sup_{(\bb_1, \hdots, \bb_r) \in \cH} \sum_{k=1}^r \bb_k^\top ( \sum_{j=1}^{t} \sum_{i=1}^{n} g_{kji}  \x_{ji})] \leq \\
		& \frac{1}{nt} \sum_{k=1}^r \mE[\norm{\sum_{j=1}^{t} \sum_{i=1}^{n} g_{kji}  \x_{ji}}] \leq  \frac{1}{nt} \sum_{k=1}^r \sqrt{\mE[\norm{\sum_{j=1}^{t} \sum_{i=1}^{n} g_{kji}  \x_{ji}}^2]} \leq \frac{1}{nt} \sum_{k=1}^r \sqrt{\sum_{j=1}^t \sum_{i=1}^n \norm{\x_{ji}}^2} \\
		& = \frac{r}{\sqrt{nt}} \sqrt{\tr(\mSigma_{\X})}.
	\end{align*}
	Further by definition the class $\cF$ as linear maps with parameters $\norm{\balpha}_2 \leq O(1)$ we obtain that $L(\cF) = O(1)$.
	We now proceed to convert this to a population quantity by noting that $\mE[\sqrt{\tr(\mSigma_{\X})}] \leq \sqrt{d \cdot \mE[\norm{\mSigma_{\X}}]} \leq O(\sqrt{d})$ for $nt \gtrsim d$ by \cref{lem:matrix_averages}.
	 \item For the complexity of learning $\cF$ in the training phase we obtain,
	\begin{align*}
		&  \EMPGaus{\h(\X)}{\cF} = \frac{1}{n} \mE[\sup_{\norm{\balpha} \leq c_1} \sum_{i=1}^n g_i \balpha^\top \B^\top \x_{ji} ] = \frac{c_1}{n} \mE[\norm{\sum_{i=1}^n g_i \B^\top \x_{ji}}] \leq \frac{c_1}{n} \sqrt{ \sum_{i=1}^n \norm{\B^\top \x_{ji}}^2} = \\
		&   \frac{c_1}{\sqrt{n}} \sqrt{\tr(\B \B^\top \mSigma_{\X_j})} = \frac{c_1}{\sqrt{n}} \sqrt{\tr(\B^\top \mSigma_{\X_j} \B)}.
	\end{align*}
	Now by the variational characterization of singular values it follows that $\max_{\B \in \cH} \frac{c_1}{\sqrt{n}} \sqrt{\tr(\B^\top \mSigma_{\X_j} \B)} \leq \frac{c_1}{n} \sqrt{\sum_{i=1}^r \sigma_i(\mSigma_{\X_j})}$
		Thus it immediately follows that, 
	\begin{align*}
		\max_{\Z \in \cZ} \frac{c_1}{\sqrt{n}} \sqrt{\tr(\mSigma_{\X_j})} = \max_{\X_j} \max_{\B \in \cH} \frac{c_1}{\sqrt{n}} \sqrt{\tr(\B^\top \mSigma_{\X_j} \B)} \leq \max_{\X_j}  \frac{c_1}{\sqrt{n}} \sqrt{\sum_{i=1}^r \sigma_i(\mSigma_{\X_j})}.
	\end{align*}
	for $j \in [t]$. We can convert this to a population quantity again by applying \cref{lem:matrix_averages} which shows  $\mE[\sqrt{\sum_{i=1}^{r} \sigma_i(\mSigma_{\X_j})}] \leq O(\sqrt{r})$ for $n \gtrsim d + \log t$. Hence $\MAXGaus{n}{\cF} \leq O(\sqrt{\frac{r}{n}})$.
	\item A nearly identical argument shows the complexity of learning $\cF$ in the testing phase is,
		\begin{align*}
			\EMPGaus{\Z_{\hh}}{\cF} = \frac{1}{m} \mE[\sup_{\norm{\balpha} \leq c_1} \sum_{i=1}^{m} \epsilon_i \balpha^\top \hB^\top \x_{(0)i} ] \leq \frac{c_1}{\sqrt{m}} \sqrt{\sum_{i=1}^r \sigma_i(\hB^\top \mSigma_{\X_{0}} \hB)}
		\end{align*} 
		Crucially, here we can apply the first result in \cref{cor:metatest_funcbernstein} which allows us to take the expectation over $\X_0$ before maximizing over $\B$. Thus applying \cref{lem:matrix_averages} as before gives the result, $\mE[\sqrt{\sum_{i=1}^r \sigma_i(\B^\top \mSigma_{\X_{0}} \B)}] \leq O(\sqrt{r})$ for $m \gtrsim r$. Hence $\MAXGaus{m}{\cF} \leq O(\sqrt{\frac{r}{m}})$.
	\end{itemize}
	This gives the first series of claims.
	
	Finally we verify that \cref{assump:regularity} holds so as to use \cref{thm:metatrain} and \cref{cor:metatest_funcbernstein} to instantiate the end-to-end guarantee. First the boundedness parameter becomes,
	\begin{align*}
		D_{\cX} = \sup_{\balpha, \B}  (\x^\top \B \balpha) \leq O(D)
	\end{align*}
	using the assumptions that $\norm{\x}_2 \leq D$, $\norm{\balpha}_2 \leq O(1)$, $\norm{\B}_2 = 1$. For the logistic loss bounds, recall $\ell(\eta; y) = y \eta - \log(1+\exp(\eta))$. Since $\abs{\nabla_{\eta} \ell(\eta; y)} = \abs{y - \frac{\exp(\eta)}{1+\exp(\eta)}} \leq 1$ it is $O(1)$-Lipschitz in its first coordinate uniformly over its second, so $L=O(1)$. Moreover, $\abs{\ell(\eta; y)} \leq O(\eta)$ where $\eta = \x^\top \B \balpha \leq \norm{\x} \leq D$ it follows the loss is uniformly bounded with parameter $O(D)$ so $B=O(D)$.
	
	Lastly, to use \cref{cor:metatest_funcbernstein} to bound the test phase error we need to compute the maximal variance term $\sigma^2=\frac{1}{m} \sup_{f \in \cF} \sum_{i=1}^m \Var(\ell(f \circ \hh(\x_i), y_i))$. Since the logistic loss $\ell(\cdot, \cdot)$ satisfies the 1-Lipschitz property uniformly we have that, $\Var(\ell(f \circ \hh(\x_i), y_i)) \leq \Var(f \circ \hh(\x_i))$ for each $i \in [m]$. Collapsing the variance we have that, 
	\begin{align*}
		& \frac{1}{m} \sup_{\balpha : \norm{\balpha}_2 \leq O(1)} \sum_{i=1}^m \Var(\x_i^\top \hB \balpha) \leq \frac{1}{m}  \sup_{\balpha : \norm{\balpha}_2 \leq O(1)} \sum_{i=1}^m (\balpha \hB)^\top \mSigma \hB \balpha \leq O(\norm{\hB \mSigma \hB}_2) \leq \\
		& O(\norm{\mSigma}) \leq O(C) = O(1)
	\end{align*}
	under our assumptions which implies that $\sigma \leq O(1)$. Assembling the previous bounds shows the transfer learning risk is bounded by,
	\begin{align*}
		& \lesssim \frac{1}{\tnu} \cdot \left( \log(nt) \cdot \left[ \sqrt{\frac{d r^2}{nt}} + \sqrt{\frac{r}{n}} \right] \right) +  \sqrt{\frac{r}{m} }  \\
		& + \left(\frac{D}{\tnu} \cdot \max \left(\frac{1}{(nt)^2}, \sqrt{\frac{\log(2/\delta)}{nt}} \right)+\sqrt{\frac{\log(2/\delta)}{m}} + D \frac{\log(2/\delta)}{m} \right).	
	\end{align*}
	with probability at least $1-2\delta$. Suppressing all logarithmic factors and using the additional condition $D \lesssim \min(dr^2, \sqrt{rm})$ guarantees the noise terms are higher-order.
\end{proof}

Recall, in the context of the two-stage ERM procedure introduced in \cref{sec:prelim} we let the design matrix and responses $y_{ji}$ for the $j$th task be $\X_j$ and $\y_j$ for $j \in [t] \cup \{0\}$, and the entire design matrix and responses concatenated over all $j \in [t]$ tasks as $\X$ and $\y$ respectively. Given a design matrix $\bX = (\x_1, \hdots, \x_N)^\top$ (comprised of mean-zero random vectors) we will let $\mSigma_{\bX} = \frac{1}{N} \bX^\top \bX$ denote its corresponding empirical covariance.

We now state a sharp, data-dependent guarantee for logistic regression.
\begin{corollary}
	If \cref{assump:linear_diverse} holds, $\h^\star(\cdot) \in \cH$, and $\cF_0 = \{ \ f \ | \ f(\x) = \balpha^\top \z, \ \balpha \in \mR^r, \ \norm{\balpha} \le c_2 \}$, then there exist constants $c_1, c_2$ such that the training tasks $f_j^{\star}$ are $(\Omega(\tnu), 0)$-diverse over $\cF_0$. Then with probability at least $1-2\delta$: 
	\begin{align*}
		 & \text{Transfer Learning Risk} \leq  \\
		 & O \Big( \frac{1}{\tnu} \cdot \Big( \log(nt) \cdot \left[ \sqrt{\frac{\tr(\mSigma_{\X}) r^2}{nt}} + \max_{j\in [t]} \sqrt{\frac{\sum_{i=1}^{r} \sigma_i(\X_j)}{n}} \right] \Big) +  \sqrt{\frac{\sum_{i=1}^{r} \sigma_i(\X_0)}{m} } \Big) \\
		& + O \Big(\frac{D}{\tnu} \cdot \max \left(\frac{1}{(nt)^2}, \sqrt{\frac{\log(4/\delta)}{nt}} \right)+\sqrt{\frac{\log(4/\delta)}{m}} + D \frac{\log(4/\delta)}{m} \Big).
	\end{align*}
	\label{cor:logistic_transfer_datadependent}
\end{corollary}
\begin{proof}[Proof of \cref{cor:logistic_transfer_datadependent}]
This follows immediately from the proof of \cref{thm:logistic_transfer} and applying \cref{cor:metatrain_emp,,cor:metatest_funcbernstein}. Merging terms and applying a union bound gives the result.
\end{proof}

The principal remaining challenge is to obtain a lower bound on the task diversity. 
\begin{lemma}
	Let \cref{assump:linear_diverse} hold in the setting of \cref{thm:logistic_transfer}. Then there exists $c_2$ such that if $c_1 \geq \frac{C}{c} c_2$ the problem is task-diverse with parameter $(\Omega(\tnu), 0)$ in the sense of \cref{def:diversity}	where $\tnu = \sigma_r(\A^\top \A/t)$.
	\label{lem:logistic_diversity}
\end{lemma}
\begin{proof}
Our first observation specializes \cref{lem:glm} to the case of logistic regression where $\Phi(\eta) = \log(1+\exp(\eta))$, $s(\sigma)=1$  with $\h(\x)=\B \x$ parametrized with $\B \in \mR^{d \times r}$ having orthonormal columns and $\f \equiv \balpha$. Throughout we also assume that $c_2$ is chosen large enough to contain all the parameters $\balphastarj$ for $j \in [t] \cup \{0\}$ and $c_1 \geq \frac{C}{c} c_2$. These conditions are consistent with the realizability assumption.

This lemma uses smoothness and (local) strong convexity to bound the task-averaged representation distance and worst-case representation difference by relating it to a result for the squared loss established in \cref{lem:square_lin_diversity}. By appealing to \cref{lem:glm} and \cref{lem:logistic_link} we have that,
\begin{align*}
	& \frac{1}{8} \mE_{\x_j}[ \exp(-\max(\abs{\hh(\x_j)^\top \halpha}, \abs{\h(\x_j)^\top \balpha})) \cdot (\hh(\x_j)^\top \halpha-\h(\x_j)^\top \balpha)^2]
 	\leq \\
 	& \mE_{\x_j, y_j}[\ell(\hat{f} \circ \hh(\x_j), y_j) - \ell(f \circ \h(\x_j), y_j)] \leq \frac{1}{8} \mE_{\x_j}[(\hh(\x_j)^\top \halpha-\h(\x_j)^\top \balpha)^2]
\end{align*}
for $\x_j, y_j \sim (\Pr_{\x}(\cdot), \Pr_{y | \x}(\cdot | f^\star_j \circ \h^\star(\x))$
We now bound each term in the task diversity,
\begin{itemize}[leftmargin=.5cm]
	\item We first bound the representation difference where $\x, y \sim (\Pr_{\x}(\cdot), \Pr_{y | \x}(\cdot | f^\star_0 \circ \h^\star(\x))$,
	\begin{align*}
		 & \dist{\hh}{\hstar} = \sup_{\balpha : \norm{\balpha}_2 \leq c_2} \inf_{\halpha : \norm{\halpha} \leq c_1} \mE_{\x, y}[\ell(\hat{f} \circ \hh, \x, y) - \ell(f^{\star}_0 \circ \hstar(\x), y)]] \leq \\
		 & \sup_{\balpha : \norm{\balpha}_2 \leq c_2} \inf_{\halpha : \norm{\halpha} \leq c_1} \frac{1}{8} \mE_{\x}[(\hh(\x)^\top \halpha-\hstar(\x)^\top \balpha)^2].
	\end{align*}
	Now for sufficiently large $c_1$, by Lagrangian duality the unconstrained minimizer of the inner optimization problem is equivalent to the constrained minimizer. In particular first note that under the assumptions of the problem there is unique unconstrained minimizer given by $\inf_{\halpha} \frac{1}{8} \mE_{\x_i}[(\hh(\x_i)^\top \halpha-\hstar(\x_i)^\top \balpha)^2] \to \halpha_{unconstrained} = -\F_{\hh \hh} \F_{\hh \h} \balpha = (\hB^\top \mSigma \hB)^{-1} (\hB^\top \mSigma \hB) \balpha$  from the proof and preamble of \cref{lem:square_lin_diversity}. Note that since $\hB$ and $\B$ have orthonormal columns it follows that $ \norm{\halpha} \leq \frac{C}{c} c_2$ since $\hB^\top \mSigma \hB$ is invertible. Thus if $c_1 \geq \frac{C}{c} c_2$, by appealing to Lagrangian duality for this convex quadratic objective with convex quadratic constraint, the unconstrained minimizer is equivalent to the constrained minimizer (since the unconstrained minimizer is contained in the constraint set). Hence leveraging the proof and result of \cref{lem:square_lin_diversity} we obtain $\sup_{\balpha : \norm{\balpha}_2 \leq c_2} \inf_{\halpha : \norm{\halpha} \leq c_1} \frac{1}{8} \mE_{\x_i}[(\hh(\x_i)^\top \halpha-\hstar(\x_i)^\top \balpha)^2] \leq \frac{c_2}{8} \sigma_{1} (\Lambda_{sc}(\h, \hh))$.
	\item We now turn out attention to controlling the average distance which we must lower bound. Here $\x_j, y_j \sim (\Pr_{\x}(\cdot), \Pr_{y | \x}(\cdot | f^\star_j \circ \h^\star(\x))$
	\begin{align*}
			& \avdist{\fstar}{\h}{\hh} = \frac{1}{t} \sum_{j=1}^{t} \inf_{\norm{\halpha} \leq c_1} \mE_{\x_j, y_j}[\ell(\hat{f} \circ \hh(\x_j), y_j) - \ell(\fjstar \circ \hstar(\x_j), y_j)]] \geq \\
			& \frac{1}{8t} \sum_{j=1}^t \mE_{\x_j}[ \exp(-\max(\abs{\hh(\x_j)^\top \halpha}, \abs{\hstar(\x_j)^\top \balphastar_j})) \cdot (\hh(\x_j)^\top \halpha-\hstar(\x_j)^\top \balphastar_j)^2]
	\end{align*}
	We will use the fact that in our logistic regression example $\h(\x_j) = \B \x_j$; in this case if $\x_j$ is $C$-subgaussian random vector in $d$ dimensions, then $\B\x_i$ is $C$-subgaussian random vector in $r$ dimensions. 
	We lower bound each term in the sum over $j$ identically and suppress the $j$ for ease of notation in the following. For fixed $j$, note the random variables $Z_1 = (\balphastar_j)^\top \B \x_i$ and $Z_2 = \halpha^\top \hB \x_i$ are subgaussian with variance parameter at most $\norm{\balphastar_j}_2^2 C^2$ and $\norm{\halpha}_2^2 C^2$ respectively. Define the event $\mathbbm{1}[E] = \mathbbm{1} [\abs{Z_1} \leq Ck  \norm{\balphastar_j} \cap   \mathbbm{1} \{\abs{Z_2} \leq Ck \norm{\halpha} ]$ for $k$ to be chosen later. We use this event to lower bound the averaged task diversity since it is a non-negative random variable, 
	\begin{align*}
		& \mE_{\x}[ \exp(-\max(\abs{\hh(\x)^\top \halpha}, \abs{\hstar(\x)^\top \balphastarj})) \cdot (\hh(\x)^\top \halpha-\hstar(\x)^\top \balphastarj)^2] \geq \\
		& \mE_{\x}[\mathbbm{1}[E] \exp(-\max(\abs{\hh(\x)^\top \halpha}, \abs{\hstar(\x)^\top \balphastarj})) \cdot (\hh(\x)^\top \halpha-\hstar(\x)^\top \balphastarj)^2] \geq \\
		& \exp(-Ck \max(c_1, c_2)) \cdot \mE_{\x}[\mathbbm{1}[E] (\hh(\x)^\top \halpha-\hstar(\x)^\top \balphastarj)^2]
	\end{align*}
	We now show that for appropriate choice of $k$, $\mE_{\x}[\mathbbm{1}[E] (\hh(\x)^\top \halpha-\hstar(\x)^\top \balphastarj)^2]$ is lower bounded by $\mE_{\x}[(\hh(\x)^\top \halpha-\hstar(\x)^\top \balphastarj)^2]$ modulo a constant factor. First write $\mE_{\x}[\mathbbm{1}[E] (\hh(\x)^\top \halpha-\hstar(\x)^\top \balphastarj)^2] = \mE_{\x}[(\hh(\x)^\top \halpha-\hstar(\x)^\top \balphastarj)^2] - \mE_{\x}[\mathbbm{1}[E^c] (\hh(\x)^\top \halpha-\hstar(\x)^\top \balphastarj)^2]$. 
	
	We upper bound the second term first using Cauchy-Schwarz,
	\begin{align*}
		\mE_{\x}[\mathbbm{1}[E^c] (\hh(\x)^\top \halpha-\hstar(\x)^\top \balphastarj)^2] \leq \sqrt{\Pr[E^c]} \sqrt{\mE_{\x} (\hh(\x)^\top \halpha-\hstar(\x)^\top \balphastarj)^4}
	\end{align*}
	Define $Z_3 = \x^\top ((\B^{\star})^\top \balphastarj - \hB^\top \halpha)$ which by definition is subgaussian with parameter at most $((\Bstar)^\top \balphastarj - \hB^\top \halpha) \mSigma ((\Bstar)^\top \balphastarj - \hB^\top \halpha) = \sigma^2$; since this condition implies L4-L2 hypercontractivity (see for example \citet[Theorem 2.6]{wainwright2019high}) we can also conclude that,
	\begin{align*}
		\sqrt{\mE_{\x} (\hh(\x)^\top \halpha-\hstar(\x)^\top \balphastarj)^4}	 \leq 10 \sigma^2 = 10 \cdot \mE_{\x} (\hh(\x)^\top \halpha-\hstar(\x)^\top \balphastarj)^2.
	\end{align*}
	 Recalling the subgaussianity of $Z_1$ and $Z_2$, from an application of Markov and Jensen's inequality,
	\begin{align*}
	\Pr[\abs{Z_1} \geq k \cdot C \norm{\balphastarj}_2] \leq \frac{\mE[Z^2]}{k^2 \cdot C^2 \norm{\balphastarj}_2} \leq \frac{1}{k^2}
	\end{align*}
	with an identical statement true for $Z_2$. Using a union bound we have that $\sqrt{\Pr[E^c]} \leq \frac{\sqrt{2}}{k}$ using these probability bounds. Hence by taking $k=30$ we can ensure that $\mE_{\x}[\mathbbm{1}[E] (\hh(\x)^\top \halpha-\hstar(\x)^\top \balphastarj)^2] \geq \frac{1}{2} \mE_{\x}[(\hh(\x)^\top \halpha-\hstar(\x)^\top \balphastarj)^2]$ by assembling the previous bounds. Finally since $c_1, c_2, C, k$ are universal constants, by definition the conclusion that,
	\begin{align*}
		& \mE_{\x}[ \exp(-\max(\abs{\hh(\x)^\top \halpha}, \abs{\hstar(\x)^\top \balphastarj})) \cdot (\hh(\x)^\top \halpha-\hstar(\x)^\top \balphastarj)^2] \geq \\
		& \Omega(\mE_{\x} (\hh(\x)^\top \halpha-\hstar(\x)^\top \balphastarj)^2) 
	\end{align*}
	follows for each $j$. Hence the average over the $t$ tasks is identically lower bounded as,
	\begin{align*}
		\Omega \left( \frac{1}{t} \sum_{j=1}^t \mE_{\x} (\hh(\x_j)^\top \halpha-\hstar(\x_j)^\top \balphastarj)^2 \right)
	\end{align*}
	
	Now using the argument from the upper bound to compute the infima since all the $\norm{\balphastarj} \leq c_2$ (and hence the constrained minimizers identical to the unconstrained minimizers for each of the $j$ terms for $c_1 \geq \frac{C}{c} c_2$) and using the proof of \cref{lem:square_lin_diversity} we conclude that,
	\begin{align*}
			\avdist{\fstar}{\hh}{\hstar} \geq \Omega(\tr(\Lambda_{sc}(\hstar, \hh) \C)).
	\end{align*}
	\end{itemize}
Combining these upper and lower bounds and concluding as in the proof of \cref{lem:square_lin_diversity} shows
\begin{align*}
	\dist{\hh}{\hstar} \leq \frac{1}{\Omega(\tnu)} \avdist{\fstar}{\hh}{\hstar}
\end{align*}
\end{proof}

Before showing the convexity-based lemmas used to control the representation differences in the loss we make a brief remark to interpret the logistic loss in the well-specified model.
\begin{remark}
	If the data generating model satisfies the logistic model conditional likelihood as in \cref{sec:logistic_reg}, for the logistic loss $\ell$ we have that,
		\begin{align*}
		& \mE_{y \sim f \circ \h(\x)}[\ell(\hat{f} \circ \hh(\x), y) - \ell(f \circ \h(\x), y)]] = 
		\mE_{\x}[\KL[\Bern(\sigma(f \circ \h(\x)) \ | \ \Bern(\sigma(\hat{f} \circ \hh(\x))]].
	\end{align*}
	simply using the fact the data is generated from the model $y \sim \Pr_{y | \x}( \cdot | f \circ \h(\x))$.
\end{remark}

To bound the task diversity we show a convexity-based lemma for general GLM/nonlinear models,
\begin{lemma}
Consider the generalized linear model for which the $\Pr_{y | \x}(\cdot)$ distribution is,
\begin{align*}
	\Pr_{y | \x}(y | \balpha^\top \h(\x)) = b(y) \exp \left(\frac{y \balpha^\top \h(\x) - \Phi(\balpha^\top \h(\x))}{s(\sigma)} \right).
\end{align*}
Then if $\sup_{p(\x) \in S(\x)} \Phi''(p(\x)) = L(\x)$ and $\inf_{p(\x) \in S(\x)} \Phi''(p(\x)) = \mu(\x)$ where $p(\x) \in S(\x) = [\hh(\x)^\top \halpha, \h(\x)^\top \balpha]$, 
\begin{align*}
	\frac{\mu(\x)}{2 s(\sigma)} (\hh(\x)^\top \halpha-\h(\x)^\top \balpha)^2
\leq \KL[\Pr_{y | \x}( \cdot | \balpha^\top \h(\x)), \Pr_{y | \x}( \cdot | \halpha^\top \hh(\x))]  \leq \frac{L(\x)}{2 s(\sigma)} 	(\hh(\x)^\top \halpha-\h(\x)^\top \balpha)^2
\end{align*}
where the $\KL$ is taken with respect to a fixed design point $\x$, and fixed feature functions $\h$, and $\hh$.
\label{lem:glm}
\end{lemma}
\begin{proof}
\begin{align*}
  & \KL[\Pr_{y | \x}( \cdot | \balpha^\top \h(\x)), \Pr_{y | \x}( \cdot | \halpha^\top \hh(\x))] = \\
  & \int{dy \ \Pr_{y | \x}( y | \balpha^\top \h(\x)) \left( \frac{ y (\h(\x)^\top \balpha-\hh(\x)^\top \halpha) }{s(\sigma)} + \frac{-\Phi(\h(\x)^\top \balpha) + \Phi(\hh(\x)^\top \halpha))}{s(\sigma)} \right)} = \\
  & \frac{1}{s(\sigma)} \left[\Phi'(\h(\x)^\top \balpha)(\h(\x)^\top \balpha-\hh(\x)^\top \halpha) -\Phi(\h(\x)^\top \balpha) + \Phi(\hh(\x)^\top \halpha) \right]
\end{align*}
since we have that $\frac{\Phi(\h(\x)^\top \balpha)}{s(\sigma)} = \log \int{ dy \ b(y) \exp(\frac{y \h(\x)^\top\balpha \rangle}{s(\sigma)})} \implies \frac{\Phi'(\h(\x)^\top \balpha)}{s(\sigma)} = \frac{\int{dy \ \Pr_{y | \x}( y | \balpha^\top \h(\x)) y}}{s(\sigma)}$ as it is the log-normalizer. Using Taylor's theorem we have that
\begin{align*}
  \Phi(\hh(\x)^\top \halpha) = \Phi \left(\h(\x)^\top \balpha \right) + \Phi'(\h(\x)^\top \balpha) (\hh(\x)^\top \halpha-\h(\x)^\top \balpha) + \frac{\Phi''(p(\x))}{2} (\hh(\x)^\top \halpha-\h(\x)^\top \balpha)^2
\end{align*}
for some intermediate $p(\x) \in [\hh(\x)^\top \halpha, \h(\x)^\top \balpha]$. Combining the previous displays we obtain that:
\begin{align*}
   \KL[\Pr_{y | \x}( \cdot | \balpha^\top \h(\x)), \Pr_{y | \x}( \cdot | \halpha^\top \hh(\x))]
   &= \frac{1}{2 s(\sigma)} \left[ \Phi''(p(\x)) (\hh(\x)^\top \halpha-\h(\x)^\top \balpha)^2 \right] 
\end{align*}
Now using the assumptions on the second derivative $\Phi''$ gives,
\begin{align*}
   \frac{\mu}{2 s(\sigma)} (\hh(\x)^\top \halpha-\h(\x)^\top \balpha)^2 \leq \frac{1}{2 s(\sigma)} \left[ \Phi''(p(\x)) (\hh(\x)^\top \halpha-\h(\x)^\top \balpha)^2 \right] \leq \frac{L}{2 s(\sigma)} (\hh(\x)^\top \halpha-\h(\x)^\top \balpha)^2
\end{align*}
\end{proof}

We now instantiate the aforementioned lemma in the setting of logistic regression.

\begin{lemma}
Consider the $\Pr_{y | \x}(\cdot)$ logistic generative model defined in \cref{sec:logistic_reg} for a general feature map $\h(\x)$. Then for this conditional generative model in the setting of \cref{lem:glm}, where $\Phi(\eta) = \log(1+\exp(\eta))$, $s(\sigma)=1$, $b(y)=1$, 
\begin{align*}
	\sup_{p(\x) \in S(\x)} \Phi''(p(\x)) \leq \frac{1}{4}
\end{align*}
and
\begin{align*}
	\inf_{p(\x) \in S(\x)} \Phi''(p(\x)) \geq \frac{1}{4} \exp(-\max(\abs{\hh(\x)^\top \halpha}, \abs{\h(\x)^\top \balpha})).
\end{align*}
for fixed $\x$. 
\label{lem:logistic_link}
\end{lemma}
\begin{proof}
A short computation shows $\Phi''(t) = \frac{e^t}{\left(e^t+1\right)^2}$. Note that the maxima of $\Phi''(t)$ over all $\mR$ occurs at $t=0$. Hence we have that, $\mE_{\x}[\sup_{p(\x) \in S(\x)} \Phi''(p(\x))] \leq \frac{1}{4}$ using a uniform upper bound. The lower bound follows by noting that 
\begin{align*}
\inf_{p(\x) \in S(\x)} \Phi''(t) = \min(\Phi''(\abs{\hh(\x_i)^\top \halpha}), \Phi''(\abs{\h(\x_i)^\top \balpha)})).
\end{align*}
For the lower bound note that for $t > 0$ that $e^{2t} \geq e^{t} \geq 1$ implies that $\frac{e^t}{(1+e^{t})^2} \geq \frac{1}{4} e^{-t}$. Since $\Phi''(t)=\Phi''(-t)$ it follows that $\Phi''(t) \geq \frac{1}{4} e^{-\abs{t}}$ for all $t \in \mR$. 
\end{proof}

Finally we include a simple auxiliary lemma to help upper bound the averages in our data-dependent bounds which relies on a simple tail bound for covariance matrices drawn from sub-gaussian ensembles (\citet[Theorem 4.7.3, Exercise 4.7.1]{vershynin2010introduction} or \citet[Theorem 6.5]{wainwright2019high}). Further recall that in \cref{def:sg} our covariate distribution is $O(1)$-sub-gaussian.
\begin{lemma}
	Let the common covariate distribution $\Pr_{\x}(\cdot)$ satisfy \cref{def:sg}. Then if $nt \gtrsim d$,
	\begin{align*}
		\mE[\norm{\mSigma_{\X}}] \leq O(1),
	\end{align*}
	if $n \gtrsim d + \log t$,
	\begin{align*}	
		\mE[\max_{j \in [t]} \norm{\mSigma_{\X_j}}] \leq O(1),
	\end{align*}
	and if $m \gtrsim r$,
	\begin{align*}
		\max_{\B \in \cH} \mE[\norm{\B^\top \mSigma_{\X_0} \B}] \leq O(1),
	\end{align*}
where $\cH$ is the set of $d \times r$ orthonormal matrices.
\label{lem:matrix_averages}
\end{lemma}
\begin{proof}
	All of these statements essentially follow by integrating a tail bound and applying the triangle inequality. For the first statement since $\mE[\norm{\mSigma_{\X}}] = \mE[\norm{\mSigma_{\X}-\mSigma}] + \norm{\mSigma} \leq O(1)$, under the conditions $nt \gtrsim d$, the result follows directly by \citet[Theorem 4.7.3]{vershynin2010introduction}.
	
	For the second by \citet[Theorem 6.5]{wainwright2019high}, $\mE[\exp(\lambda \norm{\mSigma-\mSigma}) \leq \exp(c_0 (\lambda^2/N) + 4d)]$ for all $\abs{\lambda} \leq \frac{N}{c_2}$, given a sample covariance averaged over $N$ datapoints, and universal constants $c_1, c_2$. So using a union bound alongside a tail integration since the data is i.i.d. across tasks,
	\begin{align*}
		& \mE[\max_{j \in [t]} \norm{\mSigma_{\X_j}-\mSigma}] \leq \int_{0}^{\infty} \min(1, t \Pr[\norm{\mSigma_{\X_{1}}-\mSigma}>\delta]) d \delta \leq \int_{0}^{\infty} \min(1, \exp(c_0 (\lambda^2/n) + 4d+\log t -\lambda \delta)] \leq \\
		& \int_{0}^{\infty} \min(1, \exp(4d+\log t) \cdot \exp(-c_1 \cdot n \min(\delta^2, \delta)) d \delta \leq O \left(\sqrt{\frac{d+\log t}{n}} + \frac{d+\log t}{n} \right) \leq O(1),
	\end{align*}
	via a Chernoff argument. The final inequality follows by bounding the tail integral and using the precondition $n \gtrsim d+\log t$. Centering the expectation and using the triangle inequality gives the conclusion.
	
	For the last statement the crucial observation that allows the condition $m \gtrsim r$, is that $\B^\top \x_{0i}$, for all $i \in [m]$, is by definition an $r$-dimensional $O(1)$-sub-Gaussian random vector since $\B$ is an orthonormal projection matrix. Thus an identical argument to the first statement gives the result.
\end{proof}
\subsection{Deep Neural Network Regression}
\label{app:dnn}

We first begin by assembling the results necessary to bound the Gaussian complexity of our deep neural network example. To begin we introduce a representative result 
which bounds the empirical Rademacher complexity of a deep neural network. 
\begin{theorem}[Theorem 2 adapted from \citet{golowich2017size}]
\label{thm:nn_bound}
	Let $\sigma$ be a 1-Lipschitz activation function with $\sigma(0)=0$, applied element-wise. Let $\cN$ be the class of real-valued networks of depth $K$ over the domain $\cX$ with bounded data $\norm{\x_{i}} \leq D$ for $i \in [n]$, where $\norm{\W_k}_{1, \infty} \leq M(k)$ for all $k \in [K]$. Then,
	\begin{align*}
		\EMPRad{\X}{\cN} \leq \left( \frac{2}{n} \Pi_{k=1}^K M(k) \right) \sqrt{(K+1+\log d) \cdot \max_{j \in [d]} \sum_{i=1}^n \x_{i,j}^2} \leq \frac{2D \sqrt{K+1+\log d} \cdot \Pi_{k=1}^K M(k)}{\sqrt{n}}.
	\end{align*}
	where $\x_{i, j}$ denotes the $j$-th coordinate of the vector $\x_i$ and $\X$ is an $n \times d$ design matrix (with $n$ datapoints).
\end{theorem}

With this result in hand we proceed to bound the Gaussian complexities for our deep neural network and prove \cref{thm:nn_regression}. Note that we make use of the result $\EMPRad{\X}{\cN} \leq \sqrt{\frac{\pi}{2}} \cdot \EMPGaus{\X}{\cN}$ and that $\EMPGaus{\X}{\cN} \leq 2 \sqrt{\log N} \cdot \EMPRad{\X}{\cN}$ for any function class $\cN$ when $\X$ has $N$ datapoints \citep[p. 97]{ledoux2013probability}.

\begin{proof}[Proof of \cref{thm:nn_regression}]
First note due to the task normalization conditions we can choose $c_1, c_2$ sufficiently large so that the realizability  assumption in \cref{assump:realizable} is satisfied--in particular, we can assume that $c_2$ is chosen large enough to contain parameter $\balpha^{\star}_0$ and $c_1$ large enough so that $c_1 M(K)^2 \geq c_1 c_3^2$ is larger then the norms of the parameters $\balphastarj$  for $j \in [t] $. 

Next recall that under the conditions of the result we can use \cref{lem:square_lin_diversity} to verify the task diversity condition is satisfied with parameters $(\tnu, 0)$ with $\tnu = \sigma_{r}(\A^\top \A/t) > 0$. In particular under the conditions of the theorem we can verify the well-conditioning of the feature representation with $c = \Omega(1)$ which follows by definition of the set $\cH$ and we can see that $\norm{\mE_{\x} [\hh(\x) \hstar(\x)^\top]}_2 \leq \mE_{\x} [\norm{\hh(\x)} \norm{\hstar(\x)}] \leq O(M(K)^2)$ using the norm bound from \cref{lem:nnnormbounds}. Hence under this setting we can choose $c_1$ sufficiently large so that $c_1 M(K)^2 \gtrsim \frac{M(K)^2}{c} c_2$. The condition $M(K) \gtrsim 1$ in the theorem statement is simply used to clean up the final bound.

In order to instantiate \cref{cor:ltlguarantee} we begin by bounding each of the complexity terms in the expression. First,
	\begin{itemize}[leftmargin=.5cm]
	\item For the feature learning complexity in the training phase we leverage \cref{thm:nn_bound} from \citet{golowich2017size} (which holds for scalar-valued outputs). For convenience let $\text{nn} = \frac{2D \sqrt{K+1+\log d} \cdot \Pi_{k=1}^{K} M(k)}{\sqrt{nt}}$. To bound this term we simply pull the summation over the rank $r$ outside the complexity and apply \cref{thm:nn_bound}, so
	\begin{align*}
		& \EMPGaus{\X}{\cH}=\frac{1}{nt} \mE[ \sup_{\cW_K} \sum_{l=1}^r \sum_{j=1}^{t} \sum_{i=1}^{n} g_{kji} \h_k(\x_{ji})] \leq \sum_{k=1}^r \EMPGaus{\X}{\h_k(\x_{ji})} \leq \log(nt) \cdot \sum_{k=1}^r \EMPRad{\X}{\h_k(\x_{ji})} \leq \\
		& \log(nt) \cdot r \cdot \text{nn} 
	\end{align*}
	since under the weight norm constraints (i.e. the max $\ell_1$ row norms are bounded) each component of the feature can be identically bounded. This immediately implies the population Gaussian complexity bound as the expectation over $\X$ is trivial.
	Further by definition the class $\cF$ as linear maps with parameters $\norm{\balpha}_2 \leq M(K)^2$ we obtain that $L(\cF) = O(M(K)^2)$.
	 \item For the complexity of learning $\cF$ in the training phase we obtain,
	\begin{align*}
		&  \EMPGaus{\X_j}{\cF} = \frac{1}{n} \mE_{\g}[\sup_{\balpha \in \cF} \sum_{i=1}^n g_{ji} \balpha^\top \h(\x_{ji}) ] = O \left( \frac{M(K)^2}{n} \mE_{\g}[\norm{ \sum_{i=1}^n g_{ji} \h(\x_{ji})}] \right)  \\
		& \leq O \left(\frac{M(K)^2}{n} \sqrt{\sum_{i=1}^n \norm{\h(\x_{ji})}^2} \right) \leq  O \left( \frac{M(K)^2}{\sqrt{n}} \max_i \norm{\h(\x_{ji})} \right).
	\end{align*}
	Now by appealing to the norm bounds on the feature map from \cref{lem:nnnormbounds} we have that $\max_{\h \in \cH} \max_{\X_j} \max_i \norm{\h(\x_{ji})} \lesssim M(K)$. Hence in conclusion we obtain the bound,
	\begin{align*}
		\MAXGaus{n}{\cF} \leq O \left( \frac{M(K)^3}{\sqrt{n}} \right)
	\end{align*}
	since the expectation is once again trivial.
	\item A nearly identical argument shows the complexity of learning $\cF$ in the testing phase is,
		\begin{align*}
			\EMPGaus{\X_0}{\cF} = \frac{1}{m} \mE_{\g} \left[ \sup_{\balpha : \norm{\balpha} \leq c_1} \sum_{i=1}^{m} g_i \balpha^\top \h(\x_{(0)i}) \right] \leq \frac{c_1 M(K)^3 }{\sqrt{m}} 
		\end{align*} 
		from which the conclusion follows.
	\end{itemize}
	
	Finally we verify that \cref{assump:regularity} holds so as to use \cref{cor:ltlguarantee} to instantiate the end-to-end guarantee. The boundedness parameter is,
	\begin{align*}
		D_{\cX} \leq O(M(K)^3)
	\end{align*}
	by \cref{lem:nnnormbounds} since it must be instantiated with $\balpha \in \cF$. For the $\ell_2$ loss bounds, $\ell(\eta; y) = (y-\eta)^2$. Since $\nabla_{\eta} \ell(\eta; y) = 2(y-\eta) \leq O(N+\abs{\eta}) = O(M(K)^3)$ where $\abs{\eta} \leq \abs{\balpha^\top \h(\x)} \leq O(M(K)^3)$ for $\balpha \in \cF$, $\h \in \cH$  by \cref{lem:nnnormbounds} and $N=O(1)$. So it follows the loss is Lipschitz with $L=O(M(K)^3)$. Moreover by an analogous argument, $\abs{\ell(\eta; y)} \leq O(M(K)^6)$ so it follows the loss is uniformly bounded with parameter $B=O(M(K)^6)$.
%
	
	Assembling the previous bounds shows the transfer learning risk is bounded by.
	\begin{align*}
		&  \lesssim \frac{L}{\tnu} \cdot \left( \log(nt) \cdot \left[ \log(nt) \cdot r \cdot M(K)^2 \cdot \text{nn} + \frac{M(K)^3}{\sqrt{n}} \right] \right) +   \frac{L M(K)^3}{\sqrt{m}}   \\
		& + \left(\frac{1}{\tnu} \cdot \max \left(L \cdot \frac{M(K)^3}{(nt)^2}, B \sqrt{\frac{\log(1/\delta)}{nt}} \right)+ B \sqrt{\frac{\log(1/\delta)}{m}}\right).	
	\end{align*}
	where $\text{nn} = \frac{2D \sqrt{K+1+\log d} \cdot \Pi_{k=1}^{K} M(k)}{\sqrt{nt}}$.
	Under the conditions of the result, the risk simplifies as in the theorem statement.
\end{proof}

We now state a simple result which allows us to bound the suprema of the empirical $\ell_2$ norm (i.e. the $D_{\bX}$ parameter in \cref{thm:metatrain}) and activation outputs for various neural networks.

\begin{lemma}
	Let $\hh(\x)$ be a vector-valued neural network of depth $K$ taking the form in \cref{eq:nn_feature} with each $f_j \equiv \balpha_j$ satisfying $\norm{\balpha_j} \leq A$ with bounded data $\norm{\x} \leq D$. Then the boundedness parameter in the setting of \cref{thm:metatrain} satisfies,
	\begin{align*}
			D_{\cX} \lesssim A D \cdot \Pi_{k=1}^{K} \norm{\W_{k}}_2.
	\end{align*}
	If we further assume that $\sigma(z) = \frac{e^z-e^{-z}}{e^z+e^{-z}}$ which is centered and $1$-Lipschitz (i.e. the tanh activation function), then we obtain the further bounds that,
	\begin{align*}
		\norm{\h(\x)} \leq \norm{\W_{K}}_{\infty \to 2} 
	\end{align*}
	and 
	\begin{align*}
		D_{\cX} \lesssim A \cdot \norm{\W_{K}}_{\infty \to 2} 
	\end{align*}
	which holds without requiring boundedness of $\x$. Note $\norm{\W_{K}}_{\infty \to 2}$ is the induced $\infty$ to $2$ operator norm. 
	\label{lem:nnnormbounds}
\end{lemma}
\begin{proof}	
For the purposes of induction let $\r_k(\cdot)$ denote the vector-valued output of the $k$th layer for $k \in [K]$. First note that the bound
	\begin{align*}
		& D_{\cX} \lesssim \sup_{\balpha, \h, \x} (\balpha^\top \h(\x))^2 \leq \sup_{\W_k, \x} A^2 \norm{\r_{K}}^2 
	\end{align*}
	Now, for the inductive step, $\norm{\r_{K}}^2 = \norm{\W_K \sigma(\W_{k-1} \r_{K-1})}^2 \leq \norm{\W_K}_2^2 \norm{\sigma(\W_{K-1} \r_{K-1})}^2 \leq \norm{\W_K}_2^2 \norm{\W_{K-1} \r_{K-1}}^2 \leq \norm{\W_K}_2^2 \norm{\W_{K-1}}_2^2 \norm{\r_{K-1}}^2$ where the first inequality follows because $\sigma(\cdot)$ is element-wise 1-Lipschitz and zero-centered.	
	 Recursively applying this inequality to the base case where $\r_0 = \x$ gives the conclusion after taking square roots.
	 
	 If we further assume that $\sigma(z) = \frac{e^z-e^{-z}}{e^z+e^{-z}}$ (which is centered and $1$-Lipschitz) then we can obtain the following result by simply bounding the last layer by noting that $\norm{\r_{K-1}}_{\infty} \leq 1$. Then,
	 \begin{align*}
	 	\norm{\h(\x)}^2 = \norm{\r_{K}}_2^2 = \norm{\W_{K} \r_{K-1}}_2^2 \leq \norm{\W_{K}}_{\infty \to 2}^2
	 \end{align*}
	 where $\norm{\W_{K}}_{\infty \to 2}$ is the induced $\infty$ to $2$ operator norm 
\end{proof}

We now turn to proving a task diversity lower bound applicable to general $\ell_2$ regression with general feature maps $\h(\cdot)$ under the assumptions of the $\Pr_{y | \x}$ of the generative model specified in \cref{eq:NN_class}. As our result holds only requiring $\fjstar \equiv \balphastarj$ and applies to more then neural network features we define some generic notation. 

We assume the data generating model takes the form, 
\begin{align}
y_{ji} = (\balphastarj)^\top \hstar(\x_{ji})  + \eta_{ji} \text{ for } j \in \{1, \hdots, t\}, i \in \{1, \hdots, n\} 
\label{eq:genfeatureregression}
\end{align}
for $\eta_{ji}$ with bounded second moments and independent of $\x_{ji}$.
Here the shared feature representation $\hstar(\cdot) \in \mR^{r}$ is given by a generic function. In our generic framework we can identify $\fjstar \equiv \balphastarj$ for $j \in \{1, \hdots, t \}$. As before we define the population task diversity matrix as $\A = (\balphastar_1, \hdots, \balphastar_t)^\top \in \mR^{t \times r}$, $\C=\A^\top \A/t$ and $\tnu= \sigma_r(\frac{\A^\top \A}{t})$. 
	Given two feature representations $\hh(\cdot)$ and $\hstar(\cdot)$, we can define their population covariance as,
	\begin{align*}
		\Lambda(\hh, \hstar) = 
		\begin{bmatrix} 
		\mE_{\x}[\hh(\x) \hh(\x)^\top] & \mE_{\x}[\hh(\x)\hstar(\x)^\top] \\
		\mE_{\x}[\hstar(\x)\hh(\x)^\top] & \mE_{\x}[\hstar(\x)\hstar(\x)^\top]
		\end{bmatrix} \equiv 
		\begin{bmatrix} 
		\F_{\hh \hh} & \F_{\hh \hstar} \\
		\F_{\hstar \hh} & \F_{\hstar \hstar}
		\end{bmatrix} \succeq 0
	\end{align*}
	and the generalized Schur complement of the representation of $\hstar$ with respect to $\hh$ as, 
	\begin{align*}
		\Lambda_{Sc}(\hh, \hstar) = \F_{\hstar \hstar}-\F_{\hstar \hh} (\F_{\hh \hh})^{\dagger} \F_{\hh \hstar} \succeq 0.
	\end{align*}
We now instantiate the definition of task diversity in this setting. 
We assume that the universal constants $c_2$ and $c_1$ are large-enough such that $\cF$ and $\cF_0$ contain the true parameters $\balpha_0^{\star}$ and $\balpha_j^{\star}$ respectively for the following.
\begin{lemma}
	Consider the $\Pr_{y | \x}(\cdot)$ regression model defined in \cref{eq:genfeatureregression} with the loss function $\ell(\cdot, \cdot)$ taken as the squared $\ell_2$ loss and let  \cref{assump:linear_diverse} hold. Then for this conditional generative model with $\cF = \{ \balpha : \balpha \in \mR^r \}$	 and $\cF_0 = \{ \balpha : \norm{\balpha}_2 \leq c_2 \}$ the model is $(\frac{\tnu}{c_2}, 0)$ diverse in the sense of \cref{def:diversity} and, 
		\begin{align*}
			\dist{\hh}{\hstar} = c_2 \cdot \sigma_{1}(\Lambda_{sc}(\hh, \hstar)); \quad \avdist{\fstar}{\hh}{\hstar} = \tr(\Lambda_{sc}(\hh, \hstar) \C).
		\end{align*}
	Moreover, if we assume the set of feature representations $\hh \in \cH$ in the infima over $\hh$ are well-conditioned in the sense that $\sigma_r(\mE_{\x}[\hh(\x) \hh(\x)^\top]) \geq c > 0$ and $\norm{\mE_{\x} [\hh(\x) \hstar(\x)^\top]}_2 \leq C$, then if
	$\cF = \{ \balpha : \norm{\balpha} \leq c_1 \}$, $\cF_0 = \{ \balpha : \norm{\balpha}_2 \leq c_2 \}$ and $c_1 \geq \frac{C}{c} c_2$, the same conclusions hold for sufficiently large constants $c_1, c_2$.
	\label{lem:square_lin_diversity}
\end{lemma}
\begin{proof}
		We first bound the worst-case representation difference and then the task-averaged representation difference. For convenience we let $\v(\halpha, \balpha) = \begin{bmatrix} \halpha \\ \balpha \end{bmatrix}$ in the following. First, note that under the regression model defined with the squared $\ell_2$ loss we have that, 
		\begin{align*}
			\mE_{\x, y \sim f \circ \h(\x)} \Big \{ \ell(\hf \circ \hh(\x), y) - \ell(f \circ \h(\x), y) \Big \} = \mE_{\x}[\abs{\halpha^\top \hh(\x)-\balpha^\top \h(\x)}^2]
		\end{align*}

		\begin{itemize}[leftmargin=.5cm]
		\item the worst-case representation difference between two distinct feature representations $\hh$ and $\hstar$ becomes
		\begin{align*}
			& \dist{\hh}{\hstar} = \sup_{\balpha : \norm{\balpha}_2 \leq c_2} \inf_{\halpha} \mE_{\x} \abs{\hh(\x)^\top \halpha - \hstar(\x)^\top \balpha_0}^2 = \\
			& \sup_{\balpha_0 : \norm{\balpha_0}_2 \leq c_2} \inf_{\halpha} \{ \v(\halpha, -\balpha)^\top \Lambda(\hh, \hstar) \v(\halpha, -\balpha) \} =  \sup_{\balpha_0 : \norm{\balpha_0}_2 \leq c_2} \inf_{\halpha} \{ \v(\halpha, \balpha_0)^\top \Lambda(\hh, \hstar) \v(\halpha, \balpha_0) \}.
		\end{align*}
		Recognizing the inner infima as the partial minimization of a convex quadratic form (see for example \citet[Example 3.15, Appendix A.5.4]{boyd2004convex}), we find that,
		\begin{align*}
			\inf_{\halpha} \{ \v(\halpha, \balpha_0)^\top \Lambda(\hh, \hstar) \v(\halpha, \balpha_0) \} = \balpha_0^\top \Lambda_{sc}(\hh, \hstar) \balpha_0 		\end{align*}
		 Note that in order for the minimization be finite we require $\F_{\hh \hh} \succeq 0$ and that $\F_{\hh \hstar} \balpha \in \text{range}(\F_{\hh \hh})$ -- which are both satisfied here since they are constructed as expectations over appropriate rank-one operators. In this case, a sufficient condition for $\halpha$ to be an minimizer is that $\halpha = -\F_{\hh \hh}^{\dagger} \F_{\hh \hstar} \balpha$. Finally the suprema over $\balpha$ can be computed using the variational characterization of the singular values. 
		\begin{align*}
		\sup_{\balpha_0 : \norm{\balpha_0}_2 \leq c_2} \balpha_0^\top \Lambda_{sc}(\hh, \hstar) \balpha_0 = c_2 \cdot \sigma_{1}(\Lambda_{sc}(\hh, \hstar))
		\end{align*}
		
		\item The task-averaged representation difference can be computed by similar means
		\begin{align*}
			& \avdist{\fstar}{\hh}{\hstar} = \frac{1}{t} \sum_{j=1}^{t} \inf_{\halpha} \mE_{\x} \abs{\hh(\x)^\top \halpha-\hstar(\x)^\top \balphastarj}^2 = \frac{1}{t} \sum_{j=1}^t (\balphastarj)^\top \Lambda_{sc}(\hh, \hstar) \balphastarj \\
			& =  \tr(\Lambda_{sc}(\hh, \hstar) \C)
		\end{align*}
		Note that since $\Lambda_{sc}(\hh, \hstar) \succeq 0$, and $\C \succeq 0$, by a corollary of the Von-Neumann trace inequality, we have that $\tr(\Lambda_{sc}(\hh, \hstar) \C) \geq \sum_{i=1}^r \sigma_{i}(\Lambda_{sc}(\hh, \hstar)) \sigma_{r-i+1}(\C) \geq \tr(\Lambda_{sc}(\hh, \hstar)) \sigma_{r}(\C) \geq \sigma_{1}(\Lambda_{sc}(\hh, \hstar)) \cdot \sigma_{r}(\C)$.
		\end{itemize}
		Combining the above two results we can immediately conclude that, 
		\begin{align*}
			\dist{\hh}{\hstar} = c_2 \sigma_{1} (\Lambda_{sc}(\hh, \hstar)) \leq \frac{1}{\tnu/c_2} \avdist{\fstar}{\hh}{\hstar}
		\end{align*}
	
	The second conclusion uses Lagrangian duality for the infima in both optimization problems for the worst-case and task-averaged representation differences. In particular, since the $\inf_{\halpha} \mE_{\x} \abs{\hh(\x)^\top \halpha-\hstar(\x)^\top \balpha}^2$ is a strongly-convex under the well-conditioned assumption, we have its unique minimizer is given by $\halpha = -(\F_{\hh \hh})^{-1} \F_{\hh \hstar} \balpha$; hence $\norm{\halpha} \leq \frac{C}{c} \norm{\balpha}$. Thus, if we consider the convex quadratically-constrained quadratic optimization problem $\inf_{\halpha : \norm{\halpha}_2 \leq c_0} \mE_{\x} \abs{\hh(\x)^\top \halpha-\hstar(\x)^\top \balpha}^2$ and $c_0 \geq \frac{C}{c} \norm{\balpha}$ the constraint is inactive, and the constrained optimization problem is equivalent to the unconstrained optimization problem. Hence for the choices of $\cF = \{ \balpha : \norm{\balpha} \leq c_1 \}$ and $\cF_0 = \{ \balpha : \norm{\balpha} \leq c_2 \}$, since all the $\norm{\balphastarj} \leq O(1)$ for $j \in [t] \cup \{0\}$, the infima in both the computation of the task-averaged distance and worst-case representation difference can be taken to be unconstrained for sufficiently large $c_1, c_2$. The second conclusion follows.
\end{proof}



\subsection{Index Models}
\label{app:singleindex}

We prove the general result which provides the end-to-end learning guarantee. Recall that we will use $\mSigma_{\X}$ to refer the sample covariance over the the training phase data.
\begin{proof}[Proof of \cref{thm:lipschitz}]
First by definition of the sets $\cF_0$ and $\cF$ the realizability assumption holds true. Next recall that under the conditions of the result we can use \cref{lem:regression_general_div} to verify the task diversity condition is satisfied with parameters $(\tnu, \tepsilon)$ for $\tnu \geq \frac{1}{t}$. Note in fact we have the stronger guarantee $\tnu \geq \frac{\norm{\v}_1}{\norm{\v}_{\infty}} \frac{1}{t}$ for $\v_j = \inf_{\hf \in \cF}\mE_{\x,  \eta}[L(\fjstar(\b^{\star}(\x))-\hf(\hat{\b}(\x))+\eta)]$. So if $\v$ is well spread-out given a particular learned representation $\hat{\b}$, the quantity $\tnu$ could be much larger in practice and the transfer more sample-efficient then the worst-case bound suggests.

	In order to instantiate \cref{cor:ltlguarantee} we begin by bounding each of the complexity terms in the expression. First,
	\begin{itemize}[leftmargin=.5cm]
	\item For the feature learning complexity in the training phase standard manipulations give,
	\begin{align*}
		 & \EMPGaus{\X}{\cH} \leq \frac{1}{nt} \mE \left[\sup_{\b : \norm{\b}_2 \leq W} \sum_{j=1}^t \sum_{i=1}^n g_{ji} \b^\top \x_{ji} \right] \leq \frac{W}{nt} \sqrt{\mE[\norm{\sum_{j=1}^t \sum_{i=1}^n g_{ji} \x_{ji}}_2^2]} \\
		 & \leq \frac{W}{nt} \sqrt{\sum_{j=1}^t \sum_{i=1}^n \norm{\x_{ji}}^2}  = \sqrt{\frac{W^2 \tr(\mSigma_{\X})}{nt}}
	\end{align*}
	Further by definition the class $\cF$ is $1$-Lipschitz so $L(\cF) = 1$. Taking expectations and using concavity of the $\sqrt{\cdot}$ yields the first term.
	 \item For the complexity of learning $\cF$ in the training phase we appeal to the Dudley entropy integral (see \citep[Theorem 5.22]{wainwright2019high}) and the metric entropy estimate from \citet[Lemma 6(i)]{kakade2011efficient}. First note that $N_{2, \b\X_j}(\cF, d_{2, \b\X_j}, \epsilon) \leq N(\cF, \norm{\cdot}_{\infty}, \epsilon)$, where the latter term refers to the covering number in the absolute sup-norm. By \citet[Lemma 6(i)]{kakade2011efficient}, $N(\cF, \norm{\cdot}_{\infty}, \epsilon) \leq \frac{1}{\epsilon} 2^{2DW/\epsilon}$. So  for all $0 \leq \epsilon \leq 1$,
	\begin{align*}
		&  \EMPGaus{\Z}{\cF} \lesssim 4 \epsilon + \frac{32}{\sqrt{n}} \int_{\epsilon/4}^{1} \sqrt{\log N(\cF, \norm{\cdot}_{\infty}, u)} d u \lesssim \epsilon+\frac{1}{\sqrt{n}} \int_{\epsilon/4}^{1} \sqrt{ \log \left( \frac{1}{u} \right) + \frac{2 WD }{u}} d u \\
		& \lesssim \epsilon + \frac{\sqrt{WD}}{\sqrt{n}} \int_{\epsilon/4}^{1} \frac{1}{u^{1/2}} du \lesssim \epsilon + \sqrt{\frac{WD}{n}} \cdot (2-\epsilon) \leq  O \left( \sqrt{\frac{WD}{n}} \right)
	\end{align*}
	using the inequality that $\log(\frac{1}{u}) \leq 2\frac{WD}{u}$ and taking $\epsilon=0$. This expression has no dependence on the input data or feature map so it immediately follows that,
	\begin{align*}
		\MAXGaus{n}{\cF} \leq O \left( \sqrt{\frac{WD}{n}} \right)
	\end{align*}
	\item A nearly identical argument shows the complexity of learning $\cF$ in the testing phase is,
		\begin{align*}
			\MAXGaus{m}{\cF} \leq O \left(\sqrt{\frac{WD}{m}} \right)
		\end{align*} 
	\end{itemize}
	
	Finally we verify that \cref{assump:regularity} holds so as to use \cref{cor:ltlguarantee} to instantiate the end-to-end guarantee. First the boundedness parameter becomes,
	\begin{align*}
		D_{\cX} = 1
	\end{align*}
	by definition since all the functions $f$ are bounded between $[0, 1]$.
	Again, simply by definition the $\ell_1$ norm is $1$-Lipschitz in its first coordinate uniformly over the choice of its second coordinate. Moreover as the noise $\eta_{ij}=O(1)$, the loss is uniformly bounded by $O(1)$ so $B=O(1)$.
	Assembling the previous bounds and simplifying shows the transfer learning risk is bounded by,
	\begin{align*}
		& \lesssim \frac{\log(nt)}{\tnu} \cdot \left( \sqrt{\frac{W^2 \mE_{\X}[\tr(\mSigma_{\X})]}{nt}} + \sqrt{\frac{WD}{n}} \right) + \sqrt{\frac{WD}{m}} + \frac{1}{(nt)^2} + \frac{1}{\tnu} \sqrt{\frac{\log(1/\delta)}{nt}} + \sqrt{\frac{\log(1/\delta)}{m}} + \tepsilon \\
	\end{align*}
	If we hide all logarithmic factors, we can verify the noise-terms are all higher-order to get the simplified statement in the lemma.
\end{proof}

We now introduce a generic bound to control the task diversity in a general setting. In the following recall $\cF_t = \text{conv} \{ f_1, \hdots, f_t \}$ where $f_j \in \cF$ for $j \in [t]$ where $\cF$ is a convex function class. Further, we define the $\tepsilon$-enlargement of $\cF_t$ with respect to the sup-norm by $\cF_{t, \tepsilon} = \{ f : \exists \tf \in \cF_t \text{ such that } \sup_{z} \abs{f(z)-f'(z)} \leq \tepsilon  \}$. We also assume the loss function $\ell(a, b) = L(a-b)$ for a positive, increasing function $L$ obeying a triangle inequality (i.e. a norm) for the following.

Our next results is generic and holds for all regression models of the form,
\begin{align}
    y = f(\h(\x))+\eta. \label{eq:gen_reg}
\end{align}
which encompasses the class of multi-index models.
\begin{lemma}
		In the aforementioned setting and consider the $\Pr_{y | \x}(\cdot)$ regression model defined in \cref{eq:gen_reg}. If $\cF$ is a convex function class, and $\cF_0= \cF_{t, \tepsilon}$ the model is $(\tnu, \tepsilon)$ diverse in the sense of \cref{def:diversity} for $\tnu \geq \frac{1}{t}$.
	\label{lem:regression_general_div}
\end{lemma}
\begin{proof}
	This result follows quickly from several properties of convex functions. We will use the pair $(\x, y)$ to refer to samples drawn from the generative model in \cref{eq:gen_reg}; that is $\x \sim \Pr_{\x}(\cdot), y \sim \Pr_{y | \x}(f \circ \h(\x))$. First the mapping
	\begin{align*}
	& (f, \hf) \to \mE_{\x, y} \left[ \ell(\hf \circ \hh(\x), y) - \ell(f \circ \h(\x), y) \right] = \\
	& \mE_{\x,  \eta}[L(f(\h(\x))-\hf(\hh(\x))+\eta)]-\mE_{\eta}[L(n)]
	\end{align*}
	 is a jointly convex function of $(f, \hf)$. This follows since first as an affine precomposition of a convex function, $L(f(\h(\x))-\hf(\hh(\x))+\eta)$ is convex for all $\x, \eta$, and second the expectation operator preserves convexity. Now by definition of $\cF_{t, \tepsilon}$, for all $f \in \cF_{t, \tepsilon}$ there exists $\tf \in \cF_t$ such $\sup_{z} \abs{f(z)-\tf(z)}  \leq \tepsilon$. Thus for all $f$ we have that, 
	 \begin{align*}
	 	\mE_{\x,  \eta}[L(f(\h(\x))-\hf(\hh(\x))+\eta)]-\mE_{\eta}[L(\eta)] \leq \mE_{\x,  \eta}[L(\tf(\h(\x))-\hf(\hh(\x))+\eta)]-\mE_{\eta}[L(\eta)] + \tepsilon
	 \end{align*}
	 for some $\tf \in \cF_t$. Then since partial minimization of $\hf$ over the convex set $\cF$ of this jointly convex upper bound preserves convexity, we have that the mapping from $f$ to $\inf_{\hf \in \cF}\mE_{\x,  n}[L(f(\h(\x))-\hf(\hh(\x))+\epsilon)]-\mE_{\eta}[L(\eta)]$ is a convex function of $f$. Thus, 
	  \begin{align*}
	  	\inf_{\hf \in \cF} \mE_{\x, \eta}[L(f(\h(\x))-\hf(\hh(\x))+\eta)]-\mE_{\eta}[L(\eta)] \leq \inf_{\hf \in \cF} \mE_{\x,  \eta}[L(\tf(\h(\x))-\hf(\hh(\x))+\eta)]-\mE_{\eta}[L(\eta)] + \tepsilon
	  \end{align*}
	  Now taking the suprema over $f \in \cF_{t, \tepsilon}$ gives,
	  \begin{align*}
	  	& \sup_{f \in \cF_{t, \tepsilon}} \inf_{\hf \in \cF} \mE_{\x,  \eta}[L(f(\h(\x))-\hf(\hh(\x))+\eta)]-\mE_{\eta}[L(\eta)] \leq \\
	  	& \sup_{\tf \in \cF_t} \inf_{\hf \in \cF} \mE_{\x, \eta}[L(\tf(\h(\x))-\hf(\hh(\x))+\eta)]-\mE_{\eta}[L(\eta)] + \tepsilon
	  \end{align*}
	  Finally, since the suprema of a a convex function over a convex hull generated by a finite set of points can be taken to occur at the generating set,
	  \begin{align*}
	  	\sup_{\tf \in \cF_t} \inf_{\hf \in \cF} \mE_{\x,  \eta}[L(\tf(\h(\x))-\hf(\hh(\x))+\eta)] = \max_{j \in [t]} \inf_{\hf \in \cF} \mE_{\x,  \eta}[L(\fjstar(\h(\x))-\hf(\hh(\x))+\eta)]
	  \end{align*}
	  To relate the worst-case and task-averaged representation differences, recall for a $t$-dimensional vector $\v$, $\norm{\v}_{\infty} \leq \norm{\v}_{1}$. Instantiating this with the vector with components $\v_j = \inf_{\hf \in \cF}\mE_{\x,  \eta}[L(\fjstar(\hstar(\x))-\hf(\hh(\x))+\eta)]$ and combining with the above shows that\footnote{note the $\mE_{\eta}[L(\eta)]$ terms cancel in the expressions for $\dist{\hh}{\hstar}$ and $\avdist{\fstar}{\hh}{\hstar}$.},
	  \begin{align*}
	  	\dist{\hh}{\hstar} \leq \avdist{\fstar}{\hh}{\hstar} \cdot \frac{1}{\tnu} + \epsilon 
	  \end{align*}
	 	where $\tnu \geq \frac{1}{t}$ (but might potentially be larger). Explicitly $\tnu \geq \frac{1}{t} \frac{\norm{\v}_1}{\norm{\v}_\infty}$. In the case the vector $\v$ is well-spread out over its coordinates we expect the bound $\norm{\v}_1 \geq \norm{\v}_{\infty}$ to be quite loose and $\tnu$ could be potentially much greater.
\end{proof}
Note if $\v$ is well-spread out -- intuitively the problem possesses a  problem-dependent ``uniformity" and the bound $\tnu \geq \frac{1}{t}$ is likely pessimistic. However, formalizing this notion in a clean way for nonparametric function classes considered herein seems quite difficult.

Also note the diversity bound of \cref{lem:regression_general_div} is valid for \textit{generic} functions and representations in addition to applying to a wide class of regression losses. In particular, all $p$-norms such $L(a,b) = \norm{a-b}_p$ satisfy the requisite conditions. Further only mild moments boundedness conditions are required on $\epsilon$ to ensure finiteness of the objective. 



\end{document}